\documentclass{article} 
\usepackage[utf8]{inputenc}
\usepackage[margin=1in]{geometry}
\usepackage{authblk}

\usepackage{times}  
\usepackage{helvet} 
\usepackage{courier}  
\usepackage[switch]{lineno}  
\usepackage[hyphens]{url}  
\usepackage{graphicx} 
\usepackage[backend=bibtex,sorting=none]{biblatex}
\usepackage{caption} 

\usepackage{amssymb}

\usepackage{hyperref}       

\usepackage{booktabs}       
\usepackage{amsfonts}       
\usepackage{nicefrac}       

\usepackage{bm,bbm}
\usepackage{amsmath,amsthm}
\usepackage{mathtools}

\usepackage{algorithmic}

\usepackage{float}
\usepackage{amssymb}

\usepackage{epsfig}
\usepackage{float}
\usepackage{mathrsfs}

\usepackage{subcaption}
\usepackage{longtable}
\usepackage{xtab}
\usepackage{color}
\usepackage{ifthen}
\usepackage{thmtools,thm-restate}
\usepackage[algoruled,boxed,lined]{algorithm2e}

\usepackage{subcaption}
\usepackage[thinlines]{easytable}
\usepackage{pifont}%
\usepackage{enumitem}
\usepackage{tikz}

\DeclareMathOperator*{\argmin}{arg\,min}
\DeclareMathOperator*{\E}{\mathbb{E}}

\newtheorem{assumption}{Assumption}

\newtheorem{lemma}{Lemma}
\newtheorem{theorem}{Theorem}
\newtheorem{corollary}{Corollary}
\newtheorem{proposition}{Proposition}

\newboolean{showcomments}
\setboolean{showcomments}{true}
\newcommand{\lina}[1]{  \ifthenelse{\boolean{showcomments}}
	{ \textcolor{blue}{(Lina says:  #1)}} {}  }

\DeclarePairedDelimiter\ceil{\lceil}{\rceil}
\newcommand{\R}{\mathbb R}

\newcommand{\A}{\mathcal A}

\newcommand{\F}{\mathcal F}

\newcommand{\Pb}{\mathbb P}
\newcommand{\X}{\mathbb X}

\newcommand{\one}{\mathbbm{1}}

\newcommand{\blue}[1]{\textcolor{blue}{#1}}

\newcommand{\nbf}{\noindent\textbf}
\newcommand{\nit}{\noindent\textit}

\allowdisplaybreaks

\title{Leveraging  Predictions in Smoothed Online Convex Optimization via Gradient-based Algorithms}
 
\author[1]{Yingying Li}

\author[1]{Na Li}
\affil[1]{John A. Paulson School of Engineering and Applied Sciences, Harvard University}
\date{}                  
\begin{document}
\maketitle
\begin{abstract}
We consider online convex optimization with time-varying stage costs and additional switching costs.  Since the switching costs introduce coupling across all stages, multi-step-ahead (long-term) predictions are  incorporated to improve the online performance. However,  longer-term predictions tend to suffer from lower quality. Thus, a critical question is:  \textit{how to reduce the  impact of  long-term prediction errors on the online performance?} To  address this question, we introduce a gradient-based online algorithm, Receding Horizon Inexact Gradient (RHIG), and analyze its  performance by dynamic regrets in terms of the temporal variation of the environment and the prediction errors. RHIG only considers at most $W$-step-ahead predictions to avoid being misled by worse predictions in the longer term. The optimal choice of $W$ suggested by our regret bounds depends on the tradeoff between the variation of the environment and the prediction accuracy. Additionally, we apply RHIG to a well-established stochastic  prediction error model and provide expected  regret  and  concentration bounds under correlated prediction errors. Lastly, we numerically test the performance of RHIG on  quadrotor tracking problems.

\end{abstract}


\section{Introduction}\label{sec:intro}

In this paper, we consider online convex optimization (OCO) with switching costs, also known as ``smoothed'' OCO  (SOCO) in the literature \cite{chen2015online,goel2019online,lin2019online,goel2019beyond,chen2018smoothed}. The stage costs are time-varying but the decision maker (agent) has access to  noisy predictions on the future costs.    Specifically, we consider stage cost function $f(x_t;\theta_t)$ parameterized by a time-varying parameter $\theta_t\in \Theta$. At each stage $t\in
\{1,2,\cdots, T\}$, the agent receives the predictions of the future parameters $\theta_{t\mid t-1}, \dots, \theta_{T\mid t-1}$, takes an action $x_t \in \X$, and suffers the stage cost $f(x_t;\theta_t)$ plus a switching cost $d(x_t, x_{t-1})$. The switching cost $d(x_t, x_{t-1})$  penalizes the changes in the actions between consecutive stages.   This problem  enjoys a wide range of applications.
For example, in the data center management problems \cite{lin2012online,lin2013dynamic},  the switching cost captures the switch on/off costs of the servers \cite{lin2013dynamic}, and  noisy predictions on future electricity prices and network traffic are available for the center manager \cite{wang2017multi,cortez2012multi}. Other applications include  smart building \cite{zanini2010online,li2019distributed}, robotics \cite{baca2018model}, smart grid \cite{gan2014real}, connected vehicles \cite{rios2016survey}, optimal control \cite{li2019online}, etc. 








Unlike  OCO \cite{hazan2016introduction}, the switching costs considered in SOCO introduce coupling among all stages, so  multi-step-ahead predictions are usually used for promoting the online performance. However, in most cases, predictions are not accurate,  and  longer-term  predictions tend to suffer lower quality. Therefore, it is crucial to study \textit{how to use the multi-step-ahead predictions effectively}, especially, \textit{how to reduce the impact of long-term prediction errors on the online performance.}



Recent years have witnessed a growing interest in studying SOCO with predictions. However, most literature avoids the complicated analysis on noisy multi-step-ahead predictions  by considering a rather simplified prediction model: the costs in the next $W$ stages are accurately predicted with no errors while the costs beyond the next $W$ stages are adversarial and not predictable at all \cite{lin2012online,lin2013dynamic,li2018online,li2019online,lin2019online}. This first-accurate-then-adversarial model is motivated by the fact that long-term predictions are much worse than the short-term ones, but it fails to capture the gradually increasing prediction errors as one predicts further into the future.
Several online algorithms have been proposed for this model, e.g. the optimization-based algorithm AFHC \cite{lin2012online}, the gradient-based algorithm RHGD \cite{li2018online}, etc. 
Moreover, there have been  a few  attempts to consider noisy multi-step-ahead predictions in SOCO. In particular, \cite{chen2015online} proposes  a stochastic  prediction error model to describe the correlation among prediction errors. This stochastic model generalizes stochastic filter prediction errors. Later, \cite{chen2016using} proposes an optimization-based algorithm CHC, which generalizes AFHC and MPC \cite{rawlings2017model}, and analyzes its performance based on the stochastic model in \cite{chen2015online}.


However, many important questions remain unresolved for SOCO with noisy predictions. For example, though the discussions on the stochastic model  in \cite{chen2015online,chen2016using} are insightful, there still lacks a general understanding on the effects of prediction errors on SOCO without any (stochastic  model)  assumptions. Moreover, most  methods in the literature \cite{lin2012online,chen2015online,chen2016using} require fully solving  multi-stage optimization programs at each stage; it is unclear whether any gradient-based algorithm,
which is more computationally efficient, would work for SOCO with noisy multi-step-ahead predictions. 

\nbf{Our contributions.}
In this paper, we introduce a gradient-based online algorithm Receding Horizon Inexact Gradient (RHIG). It is a straightforward extension of RHGD, which was designed for the simple first-accurate-then-adversarial prediction model in \cite{li2018online}. In RHIG, the agent can choose to  utilize only  $W\geq 0$ steps of future predictions, where $W$ is a tunable parameter for the agent. 

We first analyze the dynamic regret of  RHIG by  considering general prediction errors without any (stochastic model) assumptions. Our  regret bound depends on both the errors of the utilized predictions, i.e. $k$-step-ahead prediction errors for $k\leq W$; and the temporal variation of the environment $V_T=\sum_{t=1}^T\sup_{x\in \X}|f(x;\theta_t)-f(x;\theta_{t-1})|$. Interestingly, the regret bound shows that the optimal choice of $W$ 
depends on the tradeoff between the  variation of environment $V_T$ and the prediction errors, that is,  a large $W$  is preferred when $V_T$ is large  while a small $W$ is preferred when the  prediction errors are large. Further, the  $k$-step prediction errors have an exponentially decaying influence on the regret bound as $k$ increases, indicating that RHIG effectively reduces the negative impact of the noisy multi-step-ahead predictions.

We then  consider the stochastic prediction error model in \cite{chen2015online,chen2016using} to analyze the performance of RHIG under correlated prediction errors. We provide an expected regret bound and a concentration bound on the regret. In both bounds,  the long-term correlation among prediction errors  has an exponentially decaying effect, indicating RHIG's good performance even with strongly correlated prediction errors.

Finally, we numerically test RHIG on  online quadrotor tracking problems. Numerical experiments show that RHIG outperforms AFHC and CHC especially under larger prediction errors. Besides, we  show that RHIG  is robust to unforeseen shocks in the future.

\nbf{Additional related work:} There is   a related line of work on predictable  OCO  (without switching costs) \cite{hall2013dynamical,rakhlin2013online,rakhlin2013optimization,jadbabaie2015online}. In this case,  stage decisions are fully decoupled and  only one-step-ahead predictions are relevant. The proposed algorithms include  OMD \cite{rakhlin2013online,rakhlin2013optimization}, DMD \cite{hall2013dynamical}, AOMD \cite{jadbabaie2015online}, whose regret bounds depend on one-step prediction errors \cite{rakhlin2013optimization,jadbabaie2015online,hall2013dynamical} and $V_T$ if   dynamic regret is concerned \cite{jadbabaie2015online}.

Besides, it is worth mentioning the related online decision making problems with coupling across stages, e.g. OCO with memory \cite{anava2015online,shi2020beyond}, online optimal control \cite{agarwal2019online,li2019online,li2020online,goel2017thinking}, online Markov decision processes \cite{even2009online,li2019onlinemdp,li2019onlineicml}, etc. Leveraging inaccurate predictions in these problems is also worth exploring. 

\nbf{Notation:}  $\Pi_{\X}$ denotes the projection onto set $\X$. $\X^T=\X\times\cdots \times \X$ is a Cartesian product.  $\nabla_x $ denotes the gradient with  $x$.   $\sum_{t=0}^k a_t=0$ if $k<0$.  $\|\cdot\|_F$ and $\|\cdot \|$ are   Frobenius norm  and   $L_2$ norm.

\section{Problem Formulation}


Consider  stage cost function $f(x_t;\theta_t)$  with a time-varying parameter $\theta_t \in \Theta$ and a switching cost $d(x_t, x_{t-1})$ that penalize the changes in the actions between stages. The total cost in horizon $T$ is:.
$
C(\bm x;\bm \theta)=\sum_{t=1}^T\left[ f(x_t;\theta_t)+ d(x_t, x_{t-1})\right]$,
where  $x_t \in \X \subseteq \R^n$, $\theta_t\in \Theta \subseteq \R^p$,  and we denote $\bm x:=(x_1^\top, \dots, x_T^\top)^\top$, $\bm \theta=(\theta_1^\top, \dots, \theta_T^\top)^\top$. The  switching cost  enjoys many applications as discussed in Section~\ref{sec:intro}. 
The presence of  switching costs $d(x_t, x_{t-1})$ couples decisions  among stages. Therefore, all parameters in  horizon $T$, i.e. $\theta_1, \dots, \theta_T$, are needed to minimize $C(\bm x;\bm \theta)$. However, in practice, only predictions are available ahead of the time and the predictions are often inaccurate, especially the long-term predictions. This may lead to wrong decisions and degrade the online performance. 
In this paper, we aim at designing an online algorithm to use prediction effectively and  
unveil the unavoidable influences of the prediction errors on the online performance. 

\nbf{Prediction models.} In this paper, we denote the prediction of the future parameter $\theta_{\tau}$  obtained at the beginning of stage $t$  as $\theta_{\tau\mid t-1}$ for $t\leq \tau \leq T$.  The initial predictions $\theta_{1\mid 0}, \dots, \theta_{T\mid 0}$ are usually available before the problem starts.  
We call $\theta_{t\mid t-k}$ as $k$-step-ahead predictions of parameter $\theta_{t}$ and let $\delta_{t}(k)$ denote the $k$-step-prediction error, i.e. 
\begin{align}\label{equ: def of deltat(k)}
    \delta_t(k):=\theta_{t}-\theta_{t\mid t-k}, \quad \forall 1\leq k\leq t.
\end{align}
For notation simplicity, we define $\theta_{t\mid \tau}:=\theta_{t\mid 0}$ for $\tau \leq 0$, and thus $\delta_t(k)=\delta_t(t)$ for $k\geq t$. Further, we denote the vector of  $k$-step prediction errors of all stages as follows
\begin{align}
	\bm \delta(k)=(\delta_1(k)^\top, \dots, \delta_T(k)^\top)^\top \in\R^{pT}
,\qquad \forall\, 1\leq k \leq T.
\end{align} 
It is commonly observed that the number of lookahead steps heavily influences the prediction accuracy and in most cases long-term prediction errors are usually larger than short-term ones. 

We will first consider the general prediction errors  without additional assumptions on $\delta_t(k)$. Then, we will carry out a more insightful discussion for the case when the prediction error $\|\delta_t(k)\|$ is non-decreasing  with the number of look-ahead steps $k$.
Further, it is also commonly observed that the prediction errors are correlated. To study how the correlation among prediction errors affect the algorithm performance, we adopt the stochastic model of prediction errors  in \cite{chen2015online}. The stochastic model is a more general version of the prediction errors for Wiener filter, Kalman filter, etc. In Section 5, we will  review  this stochastic model and analyze the performance under this model.

\nbf{Protocols.} We summarize the protocols of our online problem below. 
We consider that the agent knows the function form $f(\cdot\,;\cdot)$ and $d(\cdot\, , \cdot)$ a priori. For each stage $t=1, 2, \dots, T$, the agent
\begin{itemize}
	\item  receives the predictions $\theta_{t\mid t-1}, \dots, \theta_{T\mid t-1}$ at the beginning of stage;\footnote{
	If only $W$-step-ahead predictions are received, we define $\theta_{t+\tau\mid t-1}:=\theta_{t+W-1\mid t-1}$ for $\tau\geq W$.}
	\item selects $x_t$ based on the predictions and the history, i.e. $\theta_1, \dots, \theta_{t-1}, \theta_{t\mid t-1}, \dots, \theta_{T\mid t-1}$;
	\item  suffers $f(x_t; \theta_t)+ d(x_t, x_{t-1})$ at the end of stage after true $\theta_t$ is revealed.
\end{itemize}









\nbf{Performance metrics.} This paper considers (expected) dynamic regret \cite{jadbabaie2015online}.   
The benchmark is  the optimal solution $\bm x^*$ in hindsight when $\bm \theta$ is known, i.e.
$
\bm x^*= \argmin_{\bm x\in\X^T}C(\bm x; \bm \theta)
$,
where $\bm x^*=((x_1^*)^\top, \dots, (x_T^*)^\top)^\top$. Notice that $\bm x^*$ depends on $\bm \theta$ but we omit $\bm \theta$ for brevity. Let $\bm x^\A$ denote the actions selected by the online algorithm $\A$. The dynamic regret of $\A$ with parameter $\bm \theta$ is defined as
\begin{equation}
\text{Reg}(\A)= C(\bm x^{\A};\bm \theta) -C(\bm x^*;\bm \theta )\label{equ: def dyn regret}
\end{equation} 
When considering stochastic prediction errors,  we define the expectation of the dynamic regret: $$
 \E[\text{Reg}(\A)]= \E\left[C(\bm x^{\A};\bm\theta) -C(\bm x^*;\bm \theta )\right]\label{equ: def exp dyn regret},$$
where the expectation is taken with respect to the randomness of the prediction error as well as the randomness of $\theta_t$ if applicable. 


Lastly, we  consider the following assumptions throughout this paper.  
\begin{assumption}\label{ass: strong convexity smooth}
	$f(x;\theta)$ is $\alpha$ strongly convex and $l_f$ smooth with respect to $x\in \X$ for any $\theta\in \Theta$.  $d(x,x')$ is convex and $l_d$ smooth with respect to $x,x'\in \X$.
\end{assumption}
\begin{assumption}\label{ass: grad lip cont}
	$\nabla_x f(x;\theta)$ is $h$-Lipschitz continuous with respect to $\theta$ for any $x$, i.e. 
	\begin{align*}
	& \left\|\nabla_x f(x;\theta_1)- \nabla_x f(x;\theta_2)\right\|\leq h \|\theta_1-\theta_2\|, \  \forall\, x \in \X, \ \theta_1, \theta_2 \in \Theta.
	\end{align*}
\end{assumption}
Assumption \ref{ass: strong convexity smooth} is  common in convex optimization literature \cite{nesterov2013introductory}. Assumption \ref{ass: grad lip cont} ensures a small prediction error on $\theta$ only causes a small  error in the gradient. Without such an assumption,  little can be achieved with noisy predictions. Lastly, we note that these assumptions are for the purpose of  theoretical regret analysis. The designed algorithm would apply for general convex smooth functions. 
\section{\textbf{R}eceding Horizon \textbf{I}nexact \textbf{G}radient  (RHIG)} \label{sec: online alg}





This section introduces our online algorithm Receding Horizon Inexact Gradient (RHIG).  It is based on a promising online algorithm RHGD \cite{li2018online}  designed for an over-simplified prediction model: at stage $t$, the next $W$-stage  parameters $\{\theta_{\tau}\}_{\tau=t}^{t+W-1}$ are exactly known but  parameters beyond $W$ steps are adversarial and totally unknown. We will first briefly  review RHGD
and then introduce our RHIG as an extension of RHGD to  handle the  inaccurate multi-step-ahead  predictions.  




\subsection{Preliminary: RHGD with accurate lookahead window}
RHGD is  built on the following  observation: the $k$-th iteration of offline gradient descent (GD) on the total cost $C(\bm x; \bm \theta)$ for stage variable $x_{\tau}(k)$, i.e.,
\begin{equation}\label{equ: offline GD}
\begin{aligned}
&x_{\tau}(k) = \Pi_{\X}[x_{\tau}(k-1)-\eta \nabla_{ x_{\tau}} C(\bm x(k-1);\bm \theta)], \quad \forall 1\leq \tau \leq T, \\
\text{where} \quad & \nabla_{x_{\tau}} C(\bm x;\bm \theta)=\nabla_{x_{\tau}} f(x_{\tau};\theta_{\tau})+ \nabla_{x_{\tau}}  d(x_{\tau}, x_{\tau-1})+\nabla_{x_{\tau}} d(x_{\tau+1}, x_{\tau})\one_{(\tau\leq T-1)},
\end{aligned}
\end{equation}
  only requires neighboring stage variables 
$x_{\tau-1}(k-1), x_{\tau}(k-1), x_{\tau+1}(k-1)$ and local parameter $\theta_{\tau}$, instead of all  variables $\bm x(k-1)$ and all parameters $\bm \theta$. This  observation allows  RHGD  \cite{li2018online} (Algorithm \ref{alg:RHGD}) to implement the offline gradient \eqref{equ: offline GD} for $W$ iterations by only using $\{\theta_{\tau}\}_{\tau=t}^{t+W-1}$. Specifically, at  stage $2-W\leq t \leq T$, RHGD initializes $x_{t+W}(0)$ by an oracle $\phi$ (Line 4), where  $\phi$ can be  OCO algorithms (e.g. OGD, OMD \cite{hazan2016introduction}) that compute $x_{t+W}(0)$ with   $\{\theta_{t}\}_{t=1}^{t+W-1}$.\footnote{For instance, if OGD  is used as the initialization oracle $\phi$, then $x_{t+W}(0)=x_{t+W-1}(0)-\xi_{t+W} \nabla_{x} f(x_{t+W-1}(0);\theta_{t+W-1})$, where $\xi_{t+W}$ denotes the stepsize. } If $t+W>T$, skip this step. 
Next, RHGD applies the offline GD  \eqref{equ: offline GD} to compute $x_{t+W-1}(1),x_{t+W-2}(2), \dots, x_t(W)$, which only uses  $\theta_{t+W-1}, \dots, \theta_t$ respectively (Line 5-7). RHGD skips  $x_{\tau}$ if $\tau\not \in \{1, \dots, T\}$. Finally, RHGD outputs $x_t(W)$, the $W$-th update of offline GD. 



\setlength{\intextsep}{5pt}

\begin{algorithm}\caption{Receding Horizon  Gradient Descent (RHGD)\cite{li2018online}}
	\label{alg:RHGD}
	\begin{algorithmic}[1]
		\STATE \textbf{Inputs:} Initial decision $x_0$;  stepsize $\eta$; initialization oracle $\phi$
		\STATE Let $x_1(0)=x_0$. 
		\FOR{$t=2-W, \dots, T$}
		\STATE Initialize $x_{t+W}(0)$ by   oracle $\phi$  if $t+W\leq T$.

		\FOR{$\tau=\min(t+W-1,T)$ \textbf{downto} $\max(t,1)$}
		\STATE  Update $x_{\tau}(t+W-\tau)$ by the offline GD on $x_{\tau}$ in \eqref{equ: offline GD}.
		\ENDFOR
		\STATE Output $x_t(W)$ when $1\leq t \leq T$.
		\ENDFOR

	\end{algorithmic}
\end{algorithm}



	
%
\subsection{Our algorithm: RHIG for inaccurate predictions}\label{subsec: RHIG}
\begin{algorithm}\caption{Receding Horizon Inexact Gradient (RHIG)}
	\label{alg:RHIG}
	\begin{algorithmic}[1]
		\STATE \textbf{Inputs:} {The length of the lookahead horizon: $W\geq 0$;} initial decision $x_0$; stepsize $\eta$; initialization oracle $\phi$
		\STATE Let $x_1(0)=x_0$. 
		\FOR{$t=2-W$ to $T$}
		
			\IF{$t+W\leq T$}

		\STATE Compute $x_{t+W}(0)$ by the initialization oracle $\phi$ with inexact  information.
			
		
			\ENDIF
		
		\FOR{$\tau=\min(t+W-1,T)$ \textbf{downto} $\max(t,1)$}
		\STATE Compute $x_{\tau}(t+W-\tau)$ based on the prediction $\theta_{\tau\mid t-1}$ and the inexact partial gradient:
			\begin{align}\label{equ: inexact offline GD}
		&x_\tau(k)= \Pi_{\X}[x_\tau(k-1)-\eta g_\tau(x_{\tau-1:\tau+1}(k-1);\theta_{\tau\mid t-1})], \quad \text{where }k=t+W-\tau.
		\end{align}
		\ENDFOR
		\STATE Output the decision $x_t(W)$ when $1\leq t \leq T$.
		\ENDFOR
	\end{algorithmic}
\end{algorithm}
With noisy predictions, it is natural
to use the prediction $\theta_{\tau\mid t-1}$ to estimate the future partial gradients,
\begin{align*}
    g_{\tau}(x_{\tau-1:\tau+1};\theta_{\tau\mid t-1})=\nabla_{x_{\tau}} f(x_{\tau};\theta_{\tau\mid t-1})+ \nabla_{x_{\tau}}  d(x_{\tau}, x_{\tau-1})+\nabla_{x_{\tau}} d(x_{\tau+1}, x_{\tau})\one_{(\tau\leq T-1)},
\end{align*} and then updates $x_{\tau}$ by the estimated gradients. This motivates   Receding Horizon Inexact Gradient (RHIG) in Algorithm \ref{alg:RHIG}. Compared with RHGD, RHIG has  the following  major differences.

\begin{itemize}[leftmargin=18pt]
\item (Line 1) Unlike RHGD, the lookahead horizon length $W\geq 0$ is  tunable  in RHIG. When selecting $W=0$, RHIG does not use any predictions in Line 5-7. When selecting $1\leq W\leq T$, RHIG utilizes  at most $W$-step-ahead predictions  $\{\theta_{\tau\mid t-1}\}_{\tau=t}^{t+W-1}$ in Line 5-7. Specifically, when $W=T$, RHIG utilizes all the future predictions $\{\theta_{\tau\mid t-1}\}_{\tau=t}^{T}$. Interestingly, one can also select $W>T$. In this case, RHIG not only utilizes all the  predictions    but also  conducts more computation based on the initial predictions $\{\theta_{\tau\mid 0}\}_{\tau=1}^T$ at $t\leq 0$ (recall that $\theta_{\tau\mid t-1}=\theta_{\tau\mid 0}$ when $t\leq 0$). Notably, when $W\to +\infty$, RHIG essentially solves $ \argmin_{\bm x\in \X^T} C(\bm x; \{\theta_{\tau\mid 0}\}_{\tau=1}^T)$   at $t\leq 0$ to serve as  warm starts  at $t=1$.\footnote{For more discussion  on  $W>T$, we refer the reader to our supplementary material.} The  choice of $W$ will be discussed in Section \ref{sec: regret general}-\ref{sec: stochastic}.



\item (Line 5) Notice that the  oracle $\phi$  no longer receives  $\theta_{t+W-1}$ exactly in RHIG, so OCO algorithms need to be modified here.  For example,  OGD  initializes $x_{t+W}(0)$ by prediction $\theta_{t+W-1\mid t-1}$:
\begin{align}\label{equ: OGD initialization}
x_{\tau}(0)=\Pi_{\X}[x_{\tau-1}(0)-\xi_{\tau} \nabla_{x_{\tau-1}} f(x_{\tau-1}(0); \theta_{\tau-1\mid t-1})], \quad \text{where } \tau=t+W.
\end{align}
Besides, we note that since  $\theta_{\tau\mid t-1}$ is available,  OGD \eqref{equ: OGD initialization} can also use $\theta_{\tau\mid t-1}$ to update $x_{\tau}(0)$. Similarly, OCO algorithms with predictions, e.g. (A)OMD \cite{rakhlin2013online,jadbabaie2015online}, DMD \cite{hall2015online}, can  be applied.

\item  (Line 7) Instead of  exact offline GD in RHGD, RHIG can be interpreted as inexact offline GD with prediction errors. Especially,  \eqref{equ: inexact offline GD} can be written as $ x_{\tau}(k)=
 x_{\tau}(k-1)-\eta \nabla_{ x_{\tau}} C(\bm x(k-1);  \theta_{\tau} -  \delta_{\tau}(W-k+1))$
by the definition \eqref{equ: def of deltat(k)}. More compactly, we can write RHIG updates as
\begin{align}\label{equ: RHGD as PGD}
\bm x(k)&=\Pi_{\mathbb X^T}\left[
\bm x(k-1)-\eta \nabla_{\bm x} C(\bm x(k-1); \bm \theta - \bm \delta(W-k+1))\right], \quad \forall\, 1\leq k \leq W,
\end{align} 
where  $\nabla_{\bm x} C(\bm x(k-1); \bm \theta - \bm \delta(W-k+1))$ is an inexact version of the  gradient $\nabla_{\bm x} C(\bm x(k-1); \bm \theta )$.

\end{itemize}

	\begin{figure}[t]
		\centering
		\includegraphics[scale=0.45]{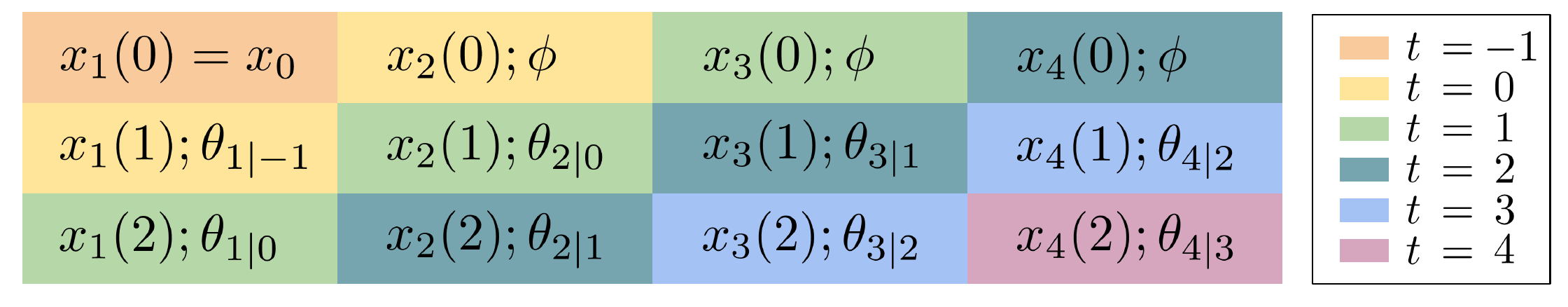}
			\caption{Example: RHIG   for  $W=2, T=4$. (Orange) at $t=-1$,  let $x_1(0)=x_0$. (Yellow) at $t=0$,   initialize $x_2(0)$ by $\phi$,  then compute $x_1(1)$ by inexact offline GD \eqref{equ: inexact offline GD} with prediction $\theta_{1\mid -1}=\theta_{1\mid 0}$. (Green) At $t=1$,  initialize $x_3(0)$ by $\phi$, and update $x_2(1)$ and  $x_1(2)$ by \eqref{equ: inexact offline GD} with $\theta_{2\mid 0}$ and $\theta_{1\mid 0}$ respectively. At $t=2$, initialize $x_4(0)$ by $\phi$, then update $x_3(1)$, $x_2(2)$ by inexact offline GD \eqref{equ: inexact offline GD} with  $\theta_{3\mid 1}$ and $\theta_{2\mid 1}$ respectively.   $t=3,4$ are similar. Notice that $ \bm x(1)=(x_1(1), \dots, x_4(1))$ is computed by inexact offline gradient with 2-step-ahead predictions, and $\bm x(2)$   by  1-step-ahead predictions.}
		\label{fig: illustration}
	\end{figure}

Though the design of RHIG is rather straightforward, both theoretical analysis and numerical experiments show promising performance of  RHIG  even under poor long-term predictions  (Section \ref{sec: regret general}-\ref{sec: simulation}). Some intuitions  are discussed below. By formula \eqref{equ: RHGD as PGD}, as the iteration number $k$ increases, RHIG employs inexact gradients with shorter-term prediction errors $\bm \delta(W-k+1)$. Since shorter-term predictions are often more accurate than  the longer-term ones, RHIG gradually utilizes more accurate gradient information as iterations go on,  reducing the optimality gap caused by inexact gradients. Further, the  longer-term prediction errors used at the first several iterations are  compressed by later  gradient updates, especially for  strongly convex costs  where GD enjoys certain contraction property.


Lastly, with a gradient-based   $\phi$ and a finite $W$,   RHIG only utilizes   gradient updates at each $t$ and is thus  more computationally efficient than  AFHC \cite{chen2015online} and CHC \cite{chen2016using} that  solve multi-stage optimization.


\section{General Regret Analysis}\label{sec: regret general}

%
%
%

This section considers general prediction errors \textit{without} stochastic model assumptions and provides  dynamic regret bounds and  discussions,\footnote{The results in this section can be extended to more general time-varying cost functions, i.e. $f_t(\cdot)$, where the prediction errors will be measured by the difference in the gradients, i.e. $\sup_{x\in \mathbb X}\|\nabla f_t(x)-\nabla f_{t\mid t-k}(x)\|$.}
before which is a helping lemma on the properties of $C(\bm x;\bm \theta)$. 


\begin{lemma}\label{lem: property C(x,theta)}
	$C(\bm x; \bm \theta)$ is $\alpha$ strongly convex and $L=l_f+2l_d$ smooth with  $\bm x\in \X^T$ for any $\bm \theta \in \Theta^T$. 
\end{lemma}


The following theorem provides a general regret bound for  RHIG with any initialization oracle $\phi$.

\begin{theorem}[General Regret Bound]\label{thm: general}
		Under Assumption~\ref{ass: strong convexity smooth}-\ref{ass: grad lip cont}, for   $W\geq 0$,  oracle $\phi$,   $\eta=\frac{1}{2L}$, we have
	\begin{align}\label{equ: general regret bdd}
\textup{Reg}(RHIG)
\leq & \ \frac{2L}{\alpha}\rho^W \textup{Reg}(\phi)+ \zeta \sum_{k=1}^{\min(W,T)}\rho^{k-1}\|\bm \delta(k)\|^2+ \one_{(W>T)} \frac{\rho^T-\rho^W}{1-\rho}\zeta \|\bm \delta(T)\|^2,
	\end{align}
	where $\rho=1-\frac{\alpha}{4L}$, $\zeta=\frac{h^2}{\alpha}+\frac{h^2}{2L}$, $\textup{Reg}(\phi)= C(\bm x(0); \bm \theta)-C(\bm x^*;\bm \theta)$ and $\bm x(0)$ is computed by   $\phi$. 
\end{theorem}



The regret bound \eqref{equ: general regret bdd} consists of three terms. The first term $\frac{2L}{\alpha} \rho^W \text{Reg}(\phi)$ depends on $\phi$. The second term $\zeta \sum_{k=1}^{\min(W,T)}\rho^{k-1}\|\bm \delta(k)\|^2$ and the third term $\one_{(W>T)} \frac{\rho^T-\rho^W}{1-\rho}\zeta \|\bm \delta(T)\|^2$ depend on  the errors of the predictions used in Algorithm \ref{alg:RHIG} (Line 5-7). Specifically, when $W\leq T$, at most $W$-step-ahead predictions are used, so the second term  involves at most $W$-step-ahead prediction errors $\{\bm \delta(k)\}_{k=1}^W$ (the third term is irrelevant). When $W>T$, RHIG uses all   predictions, so the second term includes all prediction errors $\{\bm \delta(k)\}_{k=1}^T$; besides,  RHIG conducts more computation by the initial predictions $\{\theta_{t\mid 0}\}_{t=1}^T$ at $t\leq 0$ (see Section \ref{sec: online alg}), causing the  third term on the initial prediction error $\|\bm \delta(T)\|^2$.

\nbf{An example of  $\phi$: restarted OGD \cite{besbes2015non}.} For more concrete discussions on the regret bound, we consider a specific $\phi$, restarted OGD   \cite{besbes2015non}, as  reviewed below. Consider an epoch size $\Delta$ and divide $T$ stages into $\ceil{T/\Delta}$ epochs with size $\Delta$.  In each epoch $k$, restart  OGD \eqref{equ: OGD initialization} and let $\xi_t=\frac{4}{\alpha j}$ at $t=k \Delta +j$ for $1\leq j\leq \Delta$.
Similar to \cite{besbes2015non}, we define the  variation of the environment as $V_T=\sum_{t=1}^T\operatorname{sup}_{x\in \X}| f(x;\theta_t)-f(x;\theta_{t-1})|$, and consider $V_T$ is known and $1\leq V_T\leq T$.\footnote{This is without loss of generality. When $V_T$ is unknown, we can use doubling tricks and adaptive stepsizes to generate similar bounds \cite{jadbabaie2015online}. $1\leq V_T\leq T$ can be enforced  by defining a proper $\theta_0$ and by normalization.} To obtain a meaningful regret bound, we impose   Assumption \ref{ass: d <beta (x-y)^2/2}, where  condition i) is common in OCO literature \cite{besbes2015non,jadbabaie2015online,mokhtari2016online} and  condition ii) requires a small switching cost under a small change of actions.

\begin{assumption}\label{ass: d <beta (x-y)^2/2}
 i)	There exists $G>0$ such that
	$ \|\nabla_x f(x; \theta)\|\leq G, \  \forall\, x\in \X, \theta\in \Theta$. ii) There exists $\beta$ such that $0\leq d(x,x') \leq \frac{\beta}{2}\|x-x'\|^2$.\footnote{Other norms work too, only leading to different constant factors in the regret bounds.}
\end{assumption}


\begin{theorem}[Regret bound of restarted OGD]\label{thm: ogd dyn regret} 
Under Assumption \ref{ass: strong convexity smooth}-\ref{ass: d <beta (x-y)^2/2}, consider $T>2$ and $\Delta=\ceil{\sqrt{2T/V_T}}$,   the initialization   based on restarted OGD  described above satisfies the  regret bound:
	\begin{align}\label{equ: ogd regret}
	\textup{Reg}(OGD)\leq C_1 \sqrt{V_T T}\log(1+\sqrt{T/V_T})+ \frac{h^2}{\alpha} \|\bm\delta(\min(W,T))\|^2,
	\end{align}
	where $C_1=\frac{4\sqrt 2G^2}{\alpha}+\frac{32\sqrt 2\beta G^2}{\alpha^2} + 20$.

\end{theorem}


Notice that restarted OGD's regret bound \eqref{equ: ogd regret} consists of two terms: the first term $C_1 \sqrt{V_T T}\log(1+\sqrt{T/V_T})$ is consistent with the original regret bound in \cite{besbes2015non} for strongly convex costs, which increases with the environment's variation  $V_T$; the second term depends on the $\min(W,T)$-step prediction error, which is intuitive since OGD \eqref{equ: OGD initialization} in our setting only has access to the inexact gradient $\nabla_{x_{s-1}} f(x_{s-1}(0); \theta_{s-1\mid s-W-1})$ predicted by the $\min(W,T)$-step-ahead prediction $\theta_{s-1\mid s-W-1}$.\footnote{We have this   error term  because we do not impose the stochastic structures of the gradient errors   in \cite{besbes2015non}.}

\begin{corollary}[RHIG with restarted OGD initialization]\label{cor: RHIG-OGD determine}
    	Under the conditions in Theorem  \ref{thm: general} and \ref{thm: ogd dyn regret}, RHIG with $\phi$ based on restarted OGD  satisfies
		\begin{align*}
	\textup{Reg}(RHIG)\leq \,&\underbrace{\rho^W \frac{2L}{\alpha} C_1\sqrt{V_T T}\log(1+\sqrt{T/V_T})}_{\textup{\blue{Part I}}}\\
	&+\! \underbrace{{\frac{2L}{\alpha}}\frac{h^2}{\alpha}\rho^{W}\!\|\bm\delta(\min(W,T))\|^2\!+\! \! \sum_{k=1}^{\min(W,T)}\! \!\zeta\rho^{k-1}\|\bm \delta(k)\|^2\!+\!\one_{(W>T)} \frac{\rho^T\!-\!\rho^W}{1-\rho}\zeta \|\bm \delta(T)\|^2}_{\textup{\blue{Part II}}}.
	\end{align*}
	where  $\rho=1-\frac{\alpha}{4L}$, $\zeta=\frac{h^2}{\alpha}+\frac{h^2}{2L}$, and $C_1$ is defined in Theorem \ref{thm: ogd dyn regret}.
\end{corollary}


{\nbf{The order of the regret bound.} The regret bound in Corollary \ref{cor: RHIG-OGD determine} consists of two parts: Part I involves the variation of the environment $V_T$; while Part II consists of the prediction errors $\{\bm \delta(k)\}_{k=1}^{\min(W,T)}$. The regret bound   can be written as  $\tilde O(\rho^W\sqrt{V_T T}+\sum_{k=1}^{\min(W,T)}\rho^{k-1} \|\bm \delta(k)\|^2)$. The prediction errors $\|\bm \delta(k)\|^2$ can be either larger or smaller than $V_T$ as mentioned in \cite{jadbabaie2015online}. When $V_T=o(T)$ and $\|\bm \delta(k)\|^2=o(T)$ for  $k\leq W$, the regret bound is $o(T)$. As a simple example of sublinear regrets, consider $\theta_{t-1}$ as the prediction of  $\theta_{t+k}$ ($k\geq 0$) at time $t$, then $\|\bm \delta(k)\|^2=O(V_T)$ under proper assumptions, so when $V_T=o(T)$, the regret  is $o(T)$.}

{\nbf{Impact of $V_T$.} The environment  variation $V_T$ only shows up in the Part I of the regret bound in Corollary \ref{cor: RHIG-OGD determine}. Fixing $V_T$, notice that Part I decays exponentially with the lookahead window $W$. This suggests that the impact of the environment variation $V_T$ on the regret bound decays exponentially when one considers a larger lookahead window $W$, which is intuitive since long-term thinking/planning allows early preparation for changes in the future and thus mitigates the negative impact of the environment variation.}

\nbf{Impact of $\bm \delta(k)$.} Part II in Corollary \ref{cor: RHIG-OGD determine}  includes the prediction error terms in \eqref{equ: ogd regret} and in Theorem \ref{thm: general}. Notably, for both $W\leq T$ and $W\geq T$, the factor in front of $\| \bm \delta(k)\|^2$ is dominated by $\rho^{k-1}$ for  $1\leq k \leq \min(W,T)$, which decays exponentially with $k$ since $0\leq \rho<1$. {This suggests that the impact of the total $k$-step-ahead prediction error $\|\bm \delta(k)\|^2$ decays exponentially with $k$, which also indicates that our RHIG (implicitly) focuses more on the short-term predictions than the long-term ones. This property benefits RHIG's performance in practice since short-term predictions are usually more accurate and reliable than the long-term ones.}



\nbf{Choices of $W$.}  The optimal choice of $W$ depends on the trade-off between $V_T$ and the prediction errors. 
For more insightful discussions, we consider non-decreasing $k$-step-ahead prediction errors, i.e.  $\|\bm \delta(k)\|\geq \|\bm \delta(k-1)\|$ for $1\leq k \leq T$ (in practice, longer-term predictions usually suffer worse quality).
 It can be shown that Part I increases with $V_T$ and Part II increases with the prediction errors. Further, as $W$ increases, Part I decreases  but  Part II increases.\footnote{All the monotonicity claims above are verified in the supplementary file and omitted here for brevity.} Thus, when Part I dominates the regret bound, i.e. $V_T$ is large when compared with the prediction errors, selecting a  large $W$  reduces the regret bound. On the contrary, when Part II dominates the regret bound, i.e. the prediction errors are large when compared with $V_T$, a small $W$ is preferred.
The choices of $W$  above are quite intuitive: when the  environment is drastically changing while  the predictions roughly follow the trends, one should use more predictions to prepare for future changes; however, with poor predictions and slowly changing  environments, one can ignore most predictions and rely   on the understanding of the current environment. Lastly, though we only consider RHIG with restarted OGD, the discussions  provide insights for other $\phi$.

\nbf{An upper and a lower bound in a special case.} Next, we consider a special case when $V_T$ is much larger than the prediction errors. It can be shown that the optimal regret is obtained when $W\to +\infty$.

%
%
%
%
%
%
%

\begin{corollary}\label{cor: special case}
	Consider non-decreasing  $k$-step-ahead prediction errors, i.e. $\|\bm \delta(k)\|^2\geq \|\bm \delta(k-1)\|^2$ for $1\leq k \leq T$.  When $ \sqrt{V_T T}\log(1+\sqrt{T/V_T}) \geq \frac{{2L} h^2\rho+{\alpha^2} \zeta}{2L C_1 (1-\rho){\alpha}}\|\bm \delta(
	T)\|^2 $, the  regret bound  is minimized by letting $W\to +\infty$. Further, when $W\to +\infty$, RHIG's regret can be bounded below.
	\begin{align*}
	\textup{Reg}(RHIG)\leq \frac{\zeta}{1-\rho} \sum_{k=1}^{T} \rho^{k-1}\|\bm \delta(k)\|^2.
	\end{align*}
\end{corollary}
Since $ \sqrt{V_T T}\log(1+\sqrt{T/V_T})$ increases with $V_T$, the condition in Corollary \ref{cor: special case}  essentially states that $V_T$ is much larger in comparison to all the prediction errors. Interestingly,  the bound in Corollary \ref{cor: special case} is not affected by $V_T$, but all prediction errors $\{\|\bm \delta(k)\|^2\}_{k=1}^T$  are involved, though the factor of $\|\bm \delta(k)\|^2$ exponentially decays with $k$. Next, we show that such dependence on $\|\bm \delta(k)\|^2$  is unavoidable.





\begin{theorem}[Lower bound for a special case]\label{thm: lower bdd} 
	 For any online algorithm $\A$, there exists nontrivial $\sum_t f(x_t;\theta_t)+d(x_t,x_{t-1})$ and predictions $\theta_{t\mid t-k}$ satisfying the condition in Corollary \ref{cor: special case}, with  parameters $\rho_0=(\frac{\sqrt{L}-\sqrt{\alpha}}{\sqrt L+\sqrt \alpha})^2$, $\zeta_0=(\frac{h(1-\sqrt{\rho_0})}{\alpha+\beta})^2\frac{\alpha(1-2\rho_0)}{2}>0$, such that 
	 the  regret  satisfies:
	\begin{align*}
	\textup{Reg}(\A) \geq \frac{\zeta_0}{(1-\rho_0)}\sum_{k=1}^T \rho_0^{k-1} \|\bm \delta(k)\|^2.
	\end{align*}
\end{theorem}
In  Theorem \ref{thm: lower bdd}, the influence of $\|\bm \delta(k)\|^2$ also decreases exponentially with $k$, though with a  smaller decay factor $\rho_0$. It is left as future work to close the gap between $\rho$ and $\rho_0$ (and between $\zeta$ and $\zeta_0$). 


\section{Stochastic Prediction Errors}\label{sec: stochastic}

%
%
%
%
%
%
%
%

In many applications, prediction errors are usually correlated. For example, the predicted market price of tomorrow usually relies on the predicted price of today, which also   depends  on the price predicted yesterday. 
Motivated by this, we adopt an insightful and general stochastic model on prediction errors, which was originally proposed in \cite{chen2015online}:
\begin{equation}\label{equ: delta stochastic model}
\delta_{t}(k)= \theta_{t}-\theta_{t\mid t-k}= \sum_{s=t-k+1}^{t}P(t-s)e_s, \quad \forall\, 1\leq k\leq t
\end{equation}
where $P(s)\in \R^{p \times q}$,   $e_1, \dots, e_T\in \R^q$ are  independent with zero mean and covariance  $R_e$. Model \eqref{equ: delta stochastic model} captures the correlation patterns described above:
the errors  $\delta_t(k)$ of different  predictions on the same parameter $\theta_t$ are correlated by sharing  common random vectors from $\{e_t, \dots, e_{t-k+1}\}$; and  the prediction errors  generated at the same stage, i.e.  $\theta_{t+k}-\theta_{t+k\mid t-1}$ for $k\geq 0$,  are correlated by sharing common random vectors from $\{e_t, \dots, e_{t+k}\}$. Notably, the coefficient matrix $P(k)$ represents the degree of correlation between the  $\delta_t(1)$ and  $ \delta_t(k)$ and between $\theta_{t}-\theta_{t\mid t-1}$ and $\theta_{t+k}-\theta_{t+k\mid t-1}$. 

As discussed in \cite{chen2015online,chen2016using}, the stochastic model \eqref{equ: delta stochastic model}  enjoys many applications, e.g. Wiener filters, Kalman filters \cite{kailath2000linear}. For instance, suppose the parameter follows a stochastic linear  system: $\theta_{t}= \gamma \theta_{t-1}+e_t$ with a given $\theta_0$ and random noise $e_t\sim N(0, 1)$.  Then $\theta_t=\gamma^k \theta_{t-k} + \sum_{s=t-k+1}^t \gamma^{t-s} e_s$,  the optimal prediction of $\theta_t$ based on  $\theta_{t-k}$ is $\theta_{t\mid t-k}= \gamma^k \theta_{t-k}$, the  prediction error $\delta_{t}(k)$ satisfies the model \eqref{equ: delta stochastic model} with $P(t-s)=\gamma^{t-s}$. A large $\gamma$ causes strong correlation among prediction errors.


Our next theorem bounds the expected regret of RHIG by the degree of correlation $\|P(k)\|_F$.

%

\begin{theorem}[Expected regret bound]\label{thm: expected regret bounds} Under Assumption \ref{ass: strong convexity smooth}-\ref{ass: grad lip cont},   $W\geq 0$,   $\eta=1/L$ and  initialization  $\phi$,  
	\begin{equation*}
	\begin{aligned}
	\E[\textup{Reg}(RHIG)]
	& \leq 	\frac{2L}{\alpha}\rho^W \E[\textup{Reg}(\phi)]+ \sum_{t=0}^{\min(W,T)-1} \zeta\|R_e\|_2 (T-t)\|P(t)\|_F^2 \frac{\rho^t-\rho^W}{1-\rho}
	\end{aligned}
	\end{equation*}
where the expectation is taken with respect to $\{e_t\}_{t=1}^T$,  $\rho=1-\frac{\alpha}{4L}$, $\zeta=\frac{h^2}{\alpha}+\frac{h^2}{2L}$.
	\end{theorem}

The first term in Theorem \ref{thm: expected regret bounds} represents the influence of $\phi$ while the second term captures the effects of the correlation. We note that the $t$-step correlation $\|P(t)\|_F^2$ decays exponentially with $t$ in the regret bound, indicating that RHIG  efficiently handles the strong correlation among prediction errors.


Next, we provide a regret bound when RHIG employs the  restarted OGD oracle as in Section \ref{sec: regret general}. Similarly, we consider a known $\E[V_T]$ and $1\leq V_T\leq T$ for technical simplicity.

\begin{corollary}[RHIG with restarted OGD]\label{cor: RHIG-OGD-stoch}
		Under Assumption \ref{ass: strong convexity smooth}-\ref{ass: d <beta (x-y)^2/2}, consider the restarted OGD  with  $\Delta=\ceil{\sqrt{2T/\E[V_T]}}$, we obtain $$
	\E[\textup{Reg}(RHIG)]\!\leq \!	\rho^W C_2\sqrt{\E[ V_T] T}\log(1\!+\!\sqrt{T/\E[V_T]})+ \!  \sum_{t=0}^{\min(W,T)-1}\! \zeta\|R_e\|_2 (T-t)\|P(t)\|_F^2 \frac{\rho^t}{1-\rho},$$ where we define ${C_2=\frac{2L C_1}{\alpha}}$ and $C_1$ is defined in Theorem \ref{thm: ogd dyn regret}.
	
\end{corollary}



  Notice that large $W$ is preferred with a large environment variation  and  weakly correlated  prediction errors, and vice versa. 

Next, we discuss the concentration property.  For simplicity, we consider  Gaussian vectors $\{e_t\}_{t=1}^T$.\footnote{Similar results can be obtained for sub-Gaussian random vectors.}

\begin{theorem}[Concentration bound]\label{thm: concentration bdd}
Consider Assumption \ref{ass: strong convexity smooth}-\ref{ass: d <beta (x-y)^2/2} and the conditions  in Corollary \ref{cor: RHIG-OGD-stoch}.  Let $\E[{\textup{RegBdd}}]$ denote the  expected regret bound in Corollary \ref{cor: RHIG-OGD-stoch} when $\E[V_T]=T$, then we have
	\begin{align*}
	\Pb(\textup{Reg}(RHIG)\geq \E[\textup{Regbdd}]+b)\leq \exp\left(-c\min\left(\frac{b^2}{K^2}, \frac{b}{K}\right)\right), \quad \forall\ b>0,
	\end{align*}
	where $K=\zeta\sum_{t=0}^{\min(T,W)-1} \|R_e\|_2 (T-t)\|P(t)\|_F^2 \frac{\rho^t}{1-\rho}$ and
	 $c$ is an absolute constant.

	%
	%
	%
\end{theorem}

Theorem \ref{thm: concentration bdd} shows that the probability of the regret being larger than the expected regret by $b>0$ decays  exponentially with $b$, indicating a nice concentration property of RHIG. Further, the concentration effect is stronger (i.e. a larger $1/K$) with a smaller degree of correlation $\|P(t)\|_F^2$.

\section{Numerical Experiments}\label{sec: simulation}



We consider online quadrotor tracking of a vertically moving target \cite{fisac2018general}. We consider (i) a high-level planning problem which is  purely online optimization  without modeling the physical dynamics; and (ii) a physical tracking problem where  simplified quadrotor dynamics are considered \cite{fisac2018general}. 
\begin{figure}
    \vspace{-0.3cm}
	\begin{subfigure}[b]{0.32\textwidth}
		\includegraphics[width=\textwidth]{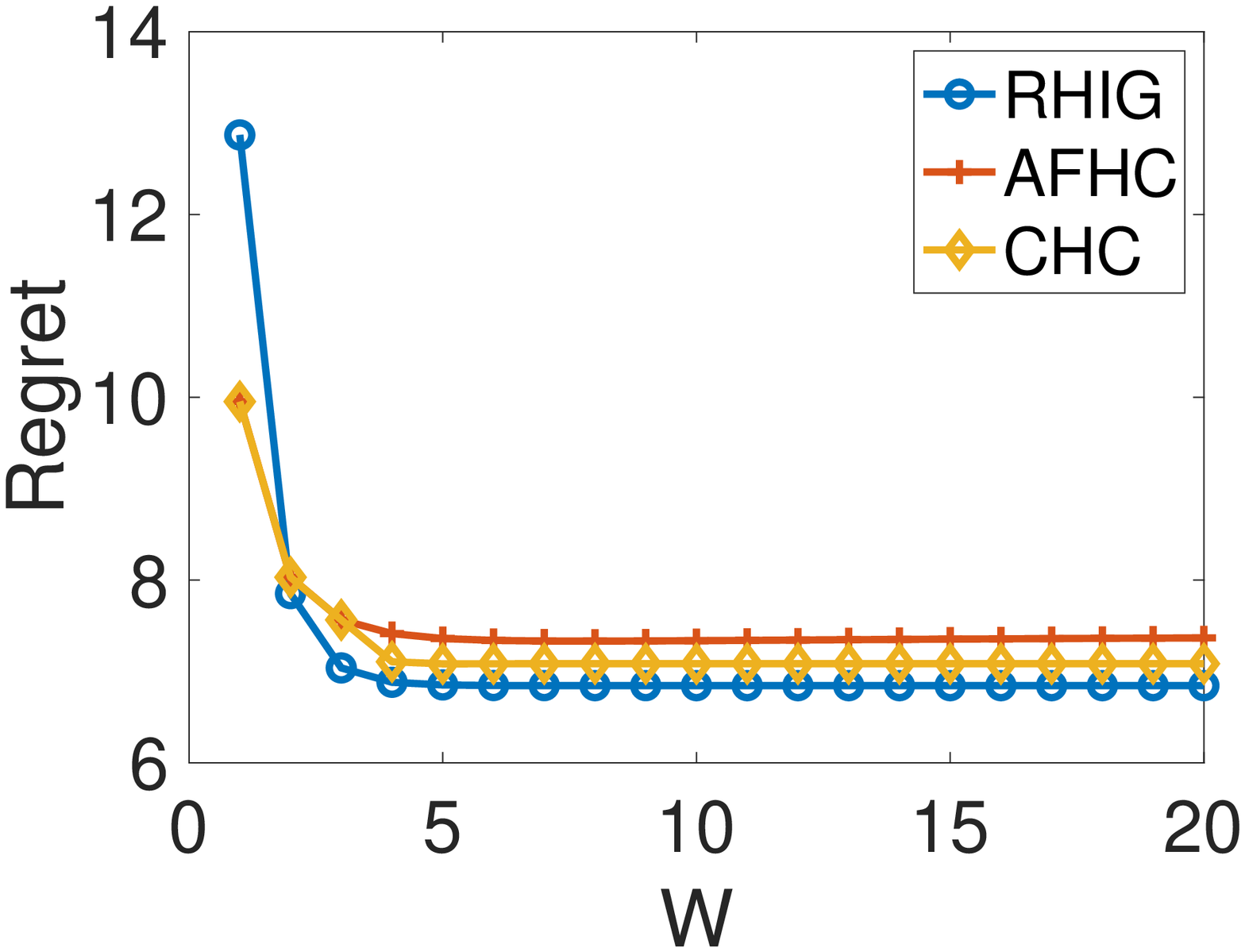}
		\caption{Small prediction errors}
	\end{subfigure} 
	\hfill
	\begin{subfigure}[b]{0.32\textwidth}
		\includegraphics[width=\textwidth]{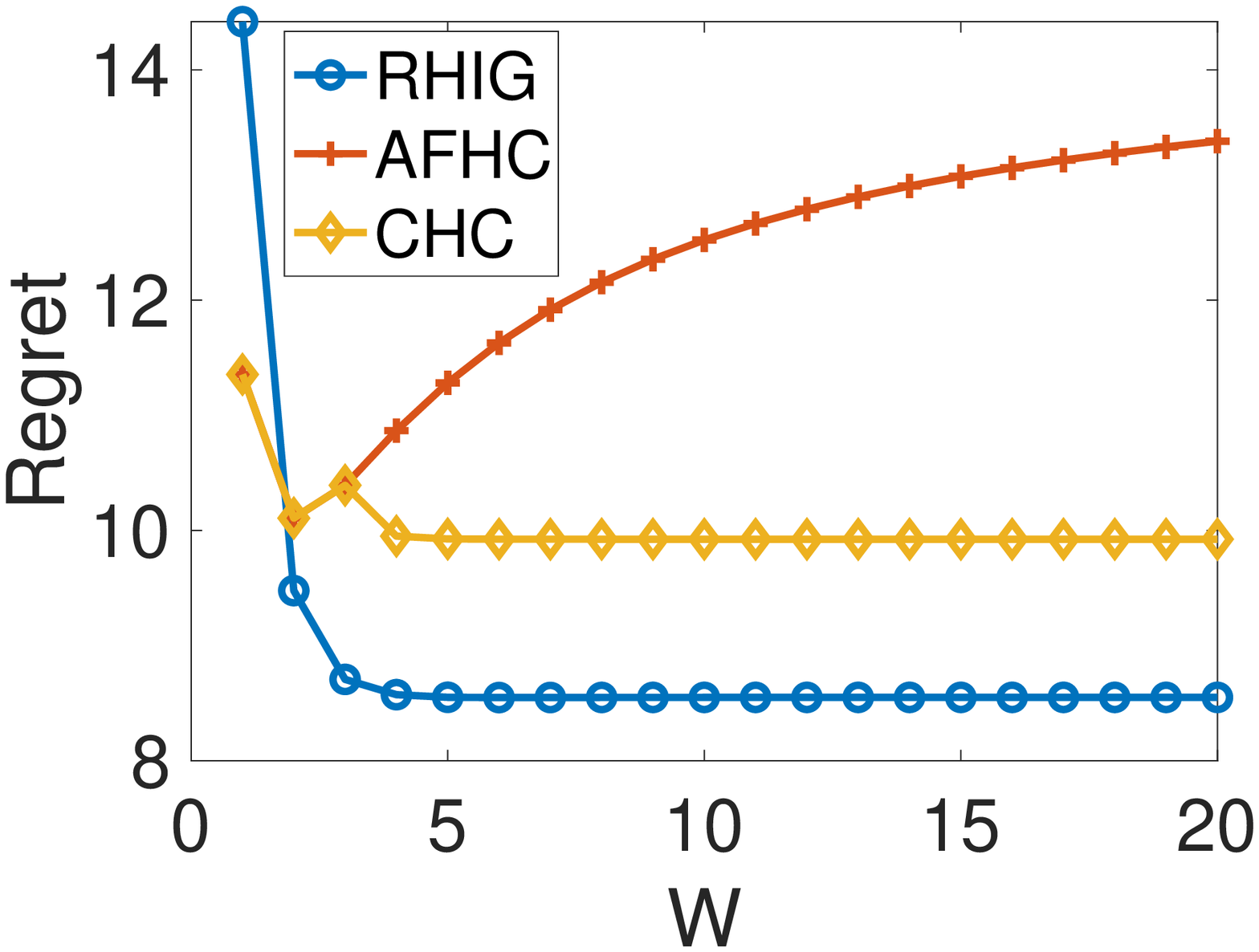}
		\caption{Large prediction errors}
	\end{subfigure}
	\hfill
\begin{subfigure}[b]{0.32\textwidth}
	\includegraphics[width=\textwidth]{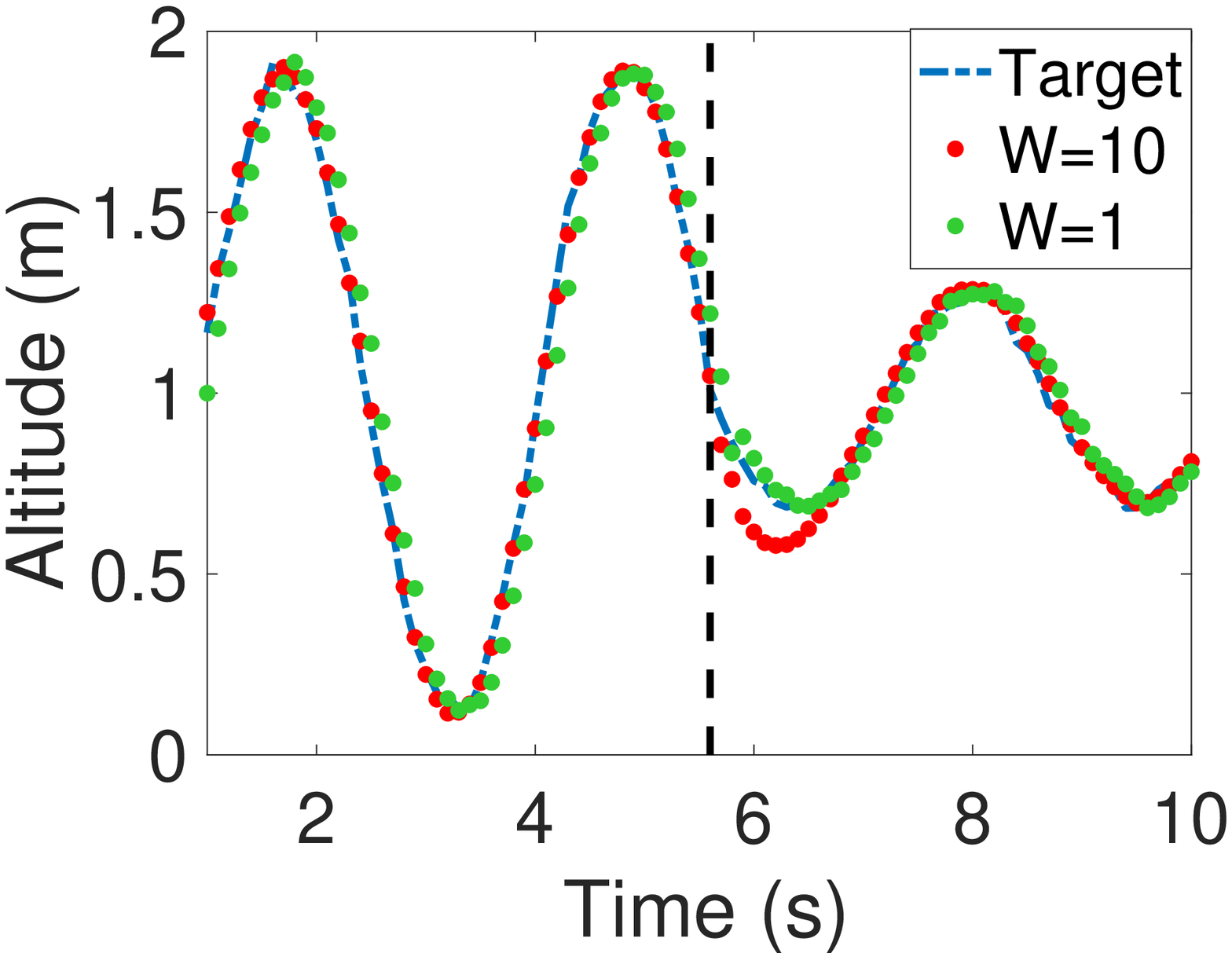}
	\caption{Quadrotor tracking}
\end{subfigure}
\setlength{\belowcaptionskip}{-10pt}
	\caption{(a) and (b):  the regrets of RHIG, AFHC  and CHC. (c): RHIG's tracking trajectories.}
	\label{fig: regret}
\end{figure}

In (i), we consider SOCO: $\min\sum_{t=1}^T\frac{1}{2}(\alpha(x_t-\theta_t)^2 + \beta(x_t-x_{t-1})^2)$, where $x_t$ is quadrotor's altitude, $\theta_t$ is  target's altitude, and $(x_t-x_{t-1})^2$ penalizes a sudden change in the quadrotor's altitude. The  target  $\theta_t$ follows: $\theta_t= y_t + d_t$, where $y_t=\gamma y_{t-1}+e_t$ is an autoregressive process with noise $e_t$ \cite{hamilton1994time}  and $d_t= a \sin(\omega t)$ is a periodic signal.  The predictions  are the sum of  $d_t$ and the optimal predictions of $y_t$. Notice that a  large  $\gamma $ indicates  worse long-term predictions. We  consider both a small $\gamma=0.3$ and a large $\gamma=0.7$ for  different levels of  errors. We  compare RHIG with AFHC \cite{lin2012online,chen2015online} and CHC  \cite{chen2016using}. (See the supplementary material for more details.) 
Figure \ref{fig: regret}(a) shows that with small  prediction errors, the three algorithms perform similarly well and  RHIG is slightly better. Figure \ref{fig: regret}(b) shows that with  large prediction errors,
 RHIG significantly outperforms AFHC and CHC. {Some intuitive explanations are provided below.  Firstly,  AFHC and CHC are optimization-based methods, while our RHIG is based on gradient descent, which is known to be more robust to errors. Secondly,  RHIG implicitly reduces the impact of the (poorer-quality) long-term predictions and focuses more on the (better) short-term ones by using long-term predictions in the first several updates and then using short-term ones in later updates to refine the decisions;  while AFHC and CHC treat  predictions more equally by taking averages of the optimal solutions computed by both long-term and short-term predictions (see \cite{chen2015online,chen2016using} for more details). These two intuitive reasons may explain the better numerical performance of our RHIG when compared with AFHC and CHC.}

In (ii), we consider  a simplified second-order  model of  quadrotor vertical
flight: $ \ddot x=k_1 u-g+k_2$,
where $x, \dot x, \ddot x$ are  the  altitude, velocity and  acceleration respectively, $u$ is the control input (motor thrust command),  $g$ is the gravitational acceleration, $k_1$ and $k_2$ are physical parameters. We consider  time discretization and cost function $\sum_{t=1}^T\frac{1}{2}(\alpha(x_t-\theta_t)^2 + \beta u_t^2)$.   The target $\theta_t$  follows the  process in (i), but with a sudden  change in  $d_t$ at $t_c=5.6$s, causing large prediction errors at around $t_c$, which is unknown until  $t_c$. Figure \ref{fig: regret}(c) plots the quadrotor's trajectories  generated by RHIG with $W=1,10$ and shows RHIG's nice tracking performance even when considering   physical dynamics. 
$W=10$ performs better first by using  more predictions. However, right after $t_c$, $W=1$ performs better since the poor prediction quality there  degrades the performance. Lastly, the trajectory with $W=10$  quickly returns to the desired one after  $t_c$, showing the robustness of RHIG to     prediction error shocks.

\vspace{-0.2cm}
\section{Conclusion}
\vspace{-0.1cm}
This paper studies how to leverage multi-step-ahead noisy predictions in smoothed online convex optimization. We design a gradient-based  algorithm RHIG and analyze its dynamic regret under  general prediction errors and a stochastic  prediction error model.  RHIG effectively reduces the  impact of multi-step-ahead  prediction errors. Future work includes: 1) closing the gap between the upper  and the lower bound in Section \ref{sec: regret general}; 2)  lower bounds for general cases; 3)  online control problems; 4) the convex case analysis without the strong convexity assumption;  etc.

%
%
%
%
%
%
%


\printbibliography

\appendix
\section*{Appendices}
The appendices provide   additional discussions and proofs of the theoretical results. In particular,  Appendix \ref{append: additional discussion} provides additional discussions on RHIG when $W>T$;  next, Appendix \ref{append: lem 1} provides a proof of Lemma \ref{lem: property C(x,theta)}; Appendix \ref{append: thm 1} provides a proof of Theorem \ref{thm: general}; Appendix \ref{append: ogd} provides a proof of Theorem \ref{thm: ogd dyn regret}, a proof of Corollary \ref{cor: RHIG-OGD determine}, and also proves the claimed properties of the regret bound in Section \ref{sec: regret general}; Appendix \ref{append: special case} discusses the special case and proves Corollary \ref{cor: special case} and Theorem \ref{thm: lower bdd}; Appendix \ref{append: stochastic} considers the stochastic prediction errors and proves Theorem \ref{thm: expected regret bounds}, Corollary \ref{cor: RHIG-OGD-stoch} and Theorem \ref{thm: concentration bdd}; finally, Appendix \ref{append: numerical} provides additional discussions on the numerical experiments.

\section{Additional Discussions on RHIG when W is larger than T}\label{append: additional discussion}

As mentioned in Section \ref{sec: online alg}, in RHIG, one can select the look-ahead horizon $W>T$. To further illustrate this case, we provide an example of RHIG for $T=3$ and $W=5$ below.

\begin{figure}[htp]
	\centering
	\includegraphics[scale=0.5]{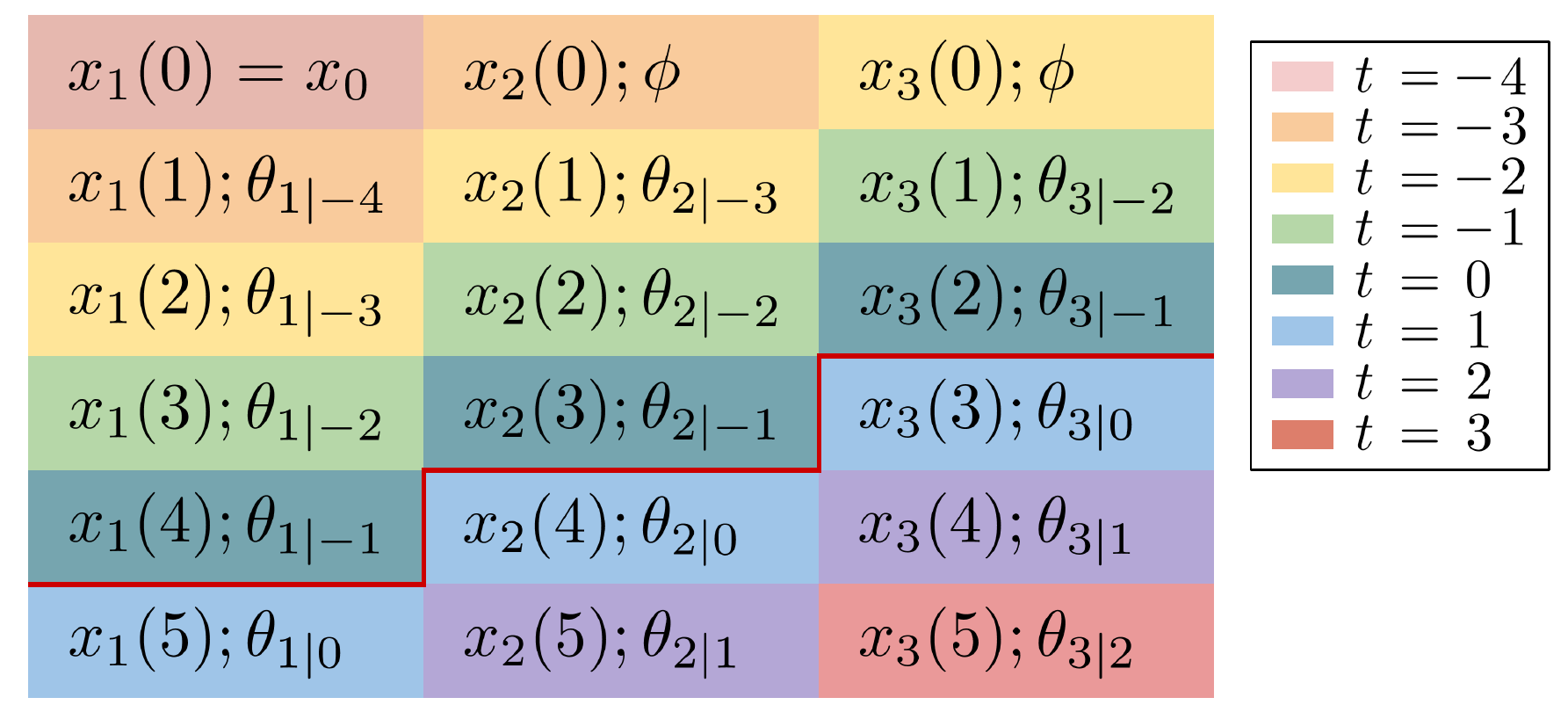}
	\caption{Example of RHIG   when   $W=5> T=3$. (Pink) At $t=1-W=-4$,  let $x_1(0)=x_0$. (Orange) At $t=-3$,   initialize $x_2(0)$ by $\phi$,  then compute $x_1(1)$ by inexact offline GD \eqref{equ: inexact offline GD} with prediction $\theta_{1\mid t-1 }=\theta_{1\mid -4}=\theta_{1\mid 0}$. (Yellow) At $t=-2$,  initialize $x_3(0)$ by $\phi$, and update $x_2(1)$ and  $x_1(2)$ by \eqref{equ: inexact offline GD} with $\theta_{2\mid -3}=\theta_{2\mid 0}$ and $\theta_{1\mid -3}=\theta_{1\mid 0}$ respectively. (Green) At $t=-1$, update $x_3(1)$, $x_2(2), x_1(3)$ by inexact offline GD \eqref{equ: inexact offline GD} with  $\theta_{3\mid -2}=\theta_{3\mid 0}$, $\theta_{2\mid -2}=\theta_{2\mid 0}$, and $\theta_{1\mid -2}=\theta_{1\mid 0}$ respectively. (Dark green) At $t=0$, update $x_3(2), x_2(3), x_1(4)$ by inexact offline GD \eqref{equ: inexact offline GD} with  $\theta_{3\mid -1}=\theta_{3\mid 0}$, $\theta_{2\mid -1}=\theta_{2\mid 0}$, and $\theta_{1\mid 12}=\theta_{1\mid 0}$ respectively. (Blue) At $t=1$, update $x_3(3), x_2(4), x_1(5)$ by inexact offline GD \eqref{equ: inexact offline GD} with  $\theta_{3\mid 0}$, $\theta_{2\mid 0}$, and $\theta_{1\mid 0}$ respectively. Then output $x_1(5)$. (Purple) At $t=2$, update $x_3(4)$ and $x_2(5)$ by inexact offline GD \eqref{equ: inexact offline GD} with  $\theta_{3\mid 1}$, $\theta_{2\mid 1}$ respectively and output $x_2(5)$. (Red) At $t=3$, update $x_3(5)$ by inexact offline GD \eqref{equ: inexact offline GD} with  $\theta_{3\mid 2}$ and output $x_3(5)$.\\
By recalling that $\theta_{t\mid \tau}=\theta_{t\mid 0}$ when $\tau<0$, we note that all the computation at $t\leq 0$ (above the red lines) is based on initial predictions $\{\theta_{1\mid 0}, \theta_{2\mid 0}, \theta_{3\mid 0}\}$. Therefore, the computed variables $x_1(4), x_2(3), x_3(2)$ at $t=0$ can be viewed as the iterated variables of offline (exact) gradient descent \eqref{equ: offline GD} under parameters $\{\theta_{1\mid 0}, \theta_{2\mid 0}, \theta_{3\mid 0}\}$ after 4,3,2 iterations respectively. 
} 

	\label{fig: illustrationW}
\end{figure} 

Further, we explain the case when $W\to +\infty$ by using the example in Figure \ref{fig: illustrationW} for $T=3$.  Similar to the discussion for Figure \ref{fig: illustrationW}, for general $W>T$,  it can be verified that the computed variables at $t=0$ are $x_1(W-1), x_2(W-2), x_3(W-3)$, which can be viewed as the iterated variables of offline (exact) gradient descent \eqref{equ: offline GD} under parameters $(\theta_{1\mid 0}, \theta_{2\mid 0}, \theta_{3\mid 0})$ after $W-1$, $W-2$, $W-3$ iterations respectively. For $W\to +\infty$,  $(x_1(W-1), x_2(W-2), x_3(W-3))$ converges to the optimal solution to $\min_{\bm x\in \X^T}C(\bm x; (\theta_{1\mid 0}, \theta_{2\mid 0}, \theta_{3\mid 0}))$. Then, at $t=1$, RHIG conducts one inexact gradient update for each $x_3(W-2), x_2(W-1), x_1(W)$ based on predictions $(\theta_{1\mid 0}, \theta_{2\mid 0}, \theta_{3\mid 0})$ and outputs $x_1(W)$. (When $W\to +\infty$, $x_1(W)$ is not updated at $t=1$ since it has converged to  $\min_{\bm x\in \X^T}C(\bm x; (\theta_{1\mid 0}, \theta_{2\mid 0}, \theta_{3\mid 0}))$.) At $t=2$,  RHIG conducts one inexact gradient update for both $x_3(W-1)$ and $x_2(W)$ based on \textit{new} prediction $\theta_{3\mid 1}$ and $\theta_{2\mid 2}$. At $t=3$, RHIG conducts one exact gradient update for $x_3(W)$ based on \textit{new} prediction $\theta_{3\mid 2}$. This explains the scenario when $W\to +\infty$.

\section{Proof of Lemma \ref{lem: property C(x,theta)}}\label{append: lem 1}
Firstly, we prove the strong convexity. Since $f(x_t;\theta_t)$ is $\alpha$-strongly convex with respect to $x_t$, $\sum_{t=1}^T f(x_t;\theta_t)$ is $\alpha$-strongly convex with respect to $\bm x=(x_1^\top, \dots, x_T^\top)^\top$. Since $d(x_t, x_{t-1})$ is convex with respect to $(x_t, x_{t-1})$, $\sum_{t=1}^T d(x_t,x_{t-1})$ is also convex with respect to $\bm x$. Consequently, $C(\bm x;\bm \theta)=\sum_{t=1}^T\left( f(x_t;\theta_t)+ d(x_t,x_{t-1})\right)$ is $\alpha$-strongly convex with respect to $\bm x$.

Next, we prove the smoothness. For any $x_t, y_t \in \X$, by the $l_f$-smoothness of $f(x_t;\theta_t)$ for all $t$, we have
\begin{align*}
 f(y_t)\leq  f(x_t)+ \langle \nabla_{x_t} f(x_t;\theta_t) , y_t-x_t\rangle + \frac{l_f}{2}\| x_t- y_t\|^2
\end{align*}
By the $l_d$-smoothness of $d(x_t, x_{t-1})$, for any $x_t, y_t, x_{t-1}, y_{t-1}\in \X$, we have
\begin{align*}
d(y_t, y_{t-1})\leq \ &d(x_t, x_{t-1}) + \langle 
\nabla_{x_t}d(x_t, x_{t-1}), 
y_t-x_t \rangle +\langle 
\nabla_{x_{t-1}}d(x_t, x_{t-1}), y_{t-1}-x_{t-1}
\rangle\\
&+ \frac{l_d}{2} (\left\| 
y_t-x_t \|^2+\|y_{t-1}-x_{t-1}
\right\|^2)
\end{align*}
for $t\geq 2$ and $d(y_1, x_0)\leq d(x_1, x_0)+ \langle \nabla_{x_1} d(x_1, x_0), y_1-x_1\rangle + \frac{l_d}{2}\|x_1-y_1\|^2$ for $t=1$.

Therefore, for any $\bm x, \bm y \in \X^T$, by summing the  smoothness inequalities above over $t=1,\dots, T$, we obtain
\begin{align*}
C(\bm y; \bm \theta)\leq C(\bm x;\bm \theta)+\langle \nabla_{\bm x} C(\bm x;\bm \theta), \bm y -\bm x\rangle + \frac{l_f+2l_d}{2}\|\bm x-\bm y\|^2
\end{align*}
where we used the fact that $\nabla_{\bm x} C(\bm x;\bm \theta)$ is composed of partial gradients $\nabla_{x_t} C(\bm x;\bm \theta)= \nabla_{x_t} f(x_t;\theta_t)+ \nabla_{x_t} d(x_t,x_{t-1})+ \one_{(t<T)}\cdot \nabla_{x_t}d(x_{t+1}, x_t)$. 



\section{Proof of Theorem \ref{thm: general}}\label{append: thm 1}
Consider the offline  optimization with parameter $\bm \theta$, i.e.
$ \min_{\bm x \in \X^T} C(\bm x;\bm \theta)$.
As mentioned in Section \ref{sec: online alg}, RHIG can be interpreted as projected gradient descent on $C(\bm x;\bm \theta)$ with inexact  gradients:
\begin{equation}\label{equ: inexact gd appendix}
\bm x(k+1)=\Pi_{\X^T} \left[ \bm x(k)-\eta \nabla_{\bm x} C(\bm x(k);\bm \theta- \bm \delta(W-k))\right]
\end{equation} 
where the exact gradient should be $ \nabla_{\bm x} C(\bm x(k);\bm \theta)$ but the parameter prediction error $\bm \delta(W-k)$ results in inexact gradient $\nabla_{\bm x} C(\bm x(k);\bm \theta- \bm \delta(W-k))$. Notice that when $W-k>T$, by our definition, we have $\bm \delta(W-k)=\bm \delta(T)$.

Consequently, the regret bound of RHIG can be proved based on the convergence analysis of the projected gradient descent with inexact gradients. We note that unlike the classic inexact gradient where the gradient errors are uniformly bounded, RHIG's inexact gradients \eqref{equ: inexact gd appendix} have different gradient errors at different iterations, thus calling for slightly different convergence analysis.

In the following, we first provide some supportive lemmas, then provide a rigorous proof of Theorem \ref{thm: general}. 

\subsection{Supportive Lemmas}
Firstly, we provide a bound on the gradient errors with respect to the errors on the parameters.
\begin{lemma}[Gradient prediction error bound]\label{lem: grad pred err bdd}
	For any true parameter $\bm \theta\in \Theta^T$ and the predicted parameter $\bm \theta' \in \Theta^T$, the error of the predicted gradient can be bounded below.
	\begin{align*}
	\| \nabla_{\bm x} C(\bm x; \bm \theta' )- \nabla_{\bm x} C(\bm x;\bm \theta)\|^2 \leq h^2\| \bm \theta'-\bm \theta\|^2, \quad \forall \ \bm x \in \Theta^T
	\end{align*}
\end{lemma}
\begin{proof}
	Firstly, we consider the gradient with respect to each stage variable $x_t$, which is provided by
	$$ \nabla_{x_t} C(\bm x; \bm \theta)= \nabla_{x_t} f(x_t; \theta_t)+ \nabla_{x_t} d(x_t, x_{t-1})+ \nabla_{x_t} d(x_{t+1}, x_t)\one_{(t\leq T-1)}$$
	Noticing that $d(x_t, x_{t-1})$ does not depend on the parameter $\bm \theta$, we obtain the prediction error bound of gradient with respect to $x_t$ as follows.
	\begin{align*}
	\|	\nabla_{x_t} C(\bm x; \bm \theta')- \nabla_{x_t} C(\bm x; \bm \theta)\|& = \| \nabla_{x_t} f(x_t; \theta_t')-\nabla_{x_t} f(x_t; \theta_t)\|\leq h\| \theta_t'-\theta_t\|
	\end{align*}
	Therefore, the prediction error of the full gradient can be bounded as follows,
	\begin{align*}
	\| \nabla_{\bm x} C(\bm x; \bm \theta' )- \nabla_{\bm x} C(\bm x;\bm \theta)\|^2& = \sum_{t=1}^T \| 	\nabla_{x_t} C(\bm x; \bm \theta')- \nabla_{x_t} C(\bm x; \bm \theta)\|^2 \leq h^2 \sum_{t=1}^T \|\theta_t' -\theta_t\|^2= h^2 \|\bm \theta'-\bm \theta\|^2
	\end{align*}
	which completes the proof.
\end{proof}

Next, we provide an equivalent characterization of the projected gradient update with respect to inexact parameters.

\begin{lemma}[A representation of inexact projected gradient updates]\label{lem: project GD= min}
	For any predicted parameter $\bm \theta'$ and any stepsize $\eta$, the  projected gradient descent with predicted parameter $ \bm x(k+1)=\Pi_{\X^T} \left[ \bm x(k)-\eta \nabla_{\bm x} C(\bm x(k);\bm \theta')\right]$ is equivalent to the following representation.
	$$ \bm x(k+1)= \argmin_{\bm x\in \X^T} \left\{ \langle \nabla_{\bm x} C(\bm x(k);\bm \theta'), \bm x-\bm x(k)\rangle + \frac{1}{2\eta}\|\bm x-\bm x(k)\|^2\right\}$$
	
\end{lemma}
\begin{proof}
	By the definition of projection, the projected gradient descent with predicted parameter is equivalent to the following.
	\begin{align*}
	\bm x(k+1)& = \argmin_{\bm x \in \X^T} \left\{ \| \bm x- \bm x(k)+\eta \nabla_{\bm x} C(\bm x(k);\bm \theta')\|^2\right\}\\
	& =  \argmin_{\bm x\in \X^T} \left\{\| \bm x-\bm x(k)\|^2+ \eta^2 \| \nabla_{\bm x} C(\bm x(k);\bm \theta')\|^2 + 2\eta\langle \nabla_{\bm x} C(\bm x(k);\bm \theta'), \bm x-\bm x(k)\rangle \right\}\\
	& = \argmin_{\bm x\in \X^T} \left\{\frac{1}{2\eta}\| \bm x-\bm x(k)\|^2+  \langle \nabla_{\bm x} C(\bm x(k);\bm \theta'), \bm x-\bm x(k)\rangle \right\}
	\end{align*}
	where the last equality uses the fact that $  \eta^2 \| \nabla_{\bm x} C(\bm x(k);\bm \theta')\|^2 $ does not depend on $\bm x$.
\end{proof}

Lastly, we provide a strong-convexity-type inequality and a smoothness-type inequality under inexact gradients. Both inequalities suffer from additional error terms caused by the  parameter prediction error.

%
%

\begin{lemma}[Strong convexity inequality with errors]\label{lem: strong convex with error C}
Consider optimization $\min_{\bm x \in \X^T} C(\bm x;\bm \theta)$. For any $\bm x, \bm y \in \X^T$,  for  any inexact parameter $\bm \theta'$ and the resulting  inexact gradient $\nabla_{\bm x} C(\bm x;\bm \theta')$, we have
\begin{align*}
C(\bm y;\bm \theta)\geq C(\bm x;\bm \theta)+ \langle \nabla_{\bm x} C(\bm x;\bm \theta'), \bm y-\bm x\rangle + \frac{\alpha}{4} \|\bm x-\bm y\|^2 - \frac{h^2}{\alpha}\|\bm \theta'-\bm \theta\|^2
\end{align*}
\end{lemma}
\begin{proof}
	By the strong convexity of $C(\bm x;\bm \theta)$, for any $\bm x, \bm y\in \X^T$ and any $\bm \theta, \bm \theta'\in \Theta^T$, we obtain the following.
	\begin{align*}
	C(\bm y; \bm \theta) &\geq C(\bm x; \bm \theta)+ \langle \nabla_{\bm x}C(\bm x;\bm \theta), \bm y-\bm x\rangle + \frac{\alpha}{2}\|\bm y -\bm x\|^2\\
	& = C(\bm x; \bm \theta)+ \langle \nabla_{\bm x}C(\bm x;\bm \theta'), \bm y-\bm x\rangle - \langle  \nabla_{\bm x}C(\bm x;\bm \theta')- \nabla_{\bm x}C(\bm x;\bm \theta), \bm y-\bm x\rangle + \frac{\alpha}{2}\|\bm y -\bm x\|^2\\
	& \geq  C(\bm x; \bm \theta)+ \langle \nabla_{\bm x}C(\bm x;\bm \theta'), \bm y-\bm x\rangle - \|  \nabla_{\bm x}C(\bm x;\bm \theta')- \nabla_{\bm x}C(\bm x;\bm \theta)\| \| \bm y-\bm x\|+ \frac{\alpha}{2}\|\bm y -\bm x\|^2\\
	& \geq C(\bm x; \bm \theta)+ \langle \nabla_{\bm x}C(\bm x;\bm \theta'), \bm y-\bm x\rangle - \frac{1}{\alpha}\|  \nabla_{\bm x}C(\bm x;\bm \theta')- \nabla_{\bm x}C(\bm x;\bm \theta)\|^2+ \frac{\alpha}{4}\|\bm y -\bm x\|^2\\
	& \geq C(\bm x; \bm \theta)+ \langle \nabla_{\bm x}C(\bm x;\bm \theta'), \bm y-\bm x\rangle - \frac{h^2}{\alpha}\| \bm \theta'-\bm \theta\|^2+ \frac{\alpha}{4}\|\bm y -\bm x\|^2
	\end{align*}
\end{proof}

\begin{lemma}[Smoothness inequality with errors]\label{lem: smooth with error C}
	Consider optimization $\min_{\bm x \in \X^T} C(\bm x;\bm \theta)$. For any $\bm x, \bm y \in \X^T$,  for  any inexact parameter $\bm \theta'$ and the resulting  inexact gradient $\nabla_{\bm x} C(\bm x;\bm \theta')$, we have
	\begin{align*}
	C(\bm y;\bm \theta)\leq C(\bm x;\bm \theta)+ \langle \nabla_{\bm x} C(\bm x;\bm \theta'), \bm y-\bm x\rangle + L \|\bm x-\bm y\|^2 + \frac{h^2}{2L}\|\bm \theta'-\bm \theta\|^2
	\end{align*}
\end{lemma}
\begin{proof}
	By the smoothness of $C(\bm x;\bm \theta)$, for any $\bm x, \bm y\in \X^T$ and any $\bm \theta, \bm \theta'\in \Theta^T$, we obtain the following.
	\begin{align*}
	C(\bm y; \bm \theta) &\leq C(\bm x; \bm \theta)+ \langle \nabla_{\bm x}C(\bm x;\bm \theta), \bm y-\bm x\rangle + \frac{L}{2}\|\bm y -\bm x\|^2\\
	& = C(\bm x; \bm \theta)+ \langle \nabla_{\bm x}C(\bm x;\bm \theta'), \bm y-\bm x\rangle + \langle  \nabla_{\bm x}C(\bm x;\bm \theta)- \nabla_{\bm x}C(\bm x;\bm \theta'), \bm y-\bm x\rangle + \frac{L}{2}\|\bm y -\bm x\|^2\\
	& \leq  C(\bm x; \bm \theta)+ \langle \nabla_{\bm x}C(\bm x;\bm \theta'), \bm y-\bm x\rangle + \|  \nabla_{\bm x}C(\bm x;\bm \theta')- \nabla_{\bm x}C(\bm x;\bm \theta)\| \| \bm y-\bm x\|+ \frac{L}{2}\|\bm y -\bm x\|^2\\
	& \leq C(\bm x; \bm \theta)+ \langle \nabla_{\bm x}C(\bm x;\bm \theta'), \bm y-\bm x\rangle +\frac{1}{2L}\|  \nabla_{\bm x}C(\bm x;\bm \theta')- \nabla_{\bm x}C(\bm x;\bm \theta)\|^2+ L\|\bm y -\bm x\|^2\\
	& \leq C(\bm x; \bm \theta)+ \langle \nabla_{\bm x}C(\bm x;\bm \theta'), \bm y-\bm x\rangle + \frac{h^2}{2L}\| \bm \theta'-\bm \theta\|^2+ L\|\bm y -\bm x\|^2
	\end{align*}
\end{proof}

\subsection{Proof of Theorem \ref{thm: general}}

	According to Algorithm \ref{alg:RHIG} and the definition of the regret, we have $\text{Reg}(RHIG)= C(\bm x(W); \bm \theta)-C(\bm x^*; \bm \theta)$ and $\text{Reg}(\phi)= C(\bm x(0); \bm \theta)-C(\bm x^*; \bm \theta)$, where $\bm x^*=\argmin_{\X^T} C(\bm x;\bm \theta)$. For notational simplicity, we denote $r_k= \|\bm x(k)-\bm x^*\|^2$. 
	
	\nit{Step 1: bound $\textup{Reg}(RHIG)$ with $r_{W-1}$.} 
	\begin{align*}
	C(\bm x(W); \bm \theta) \leq \ & C(\bm x(W-1); \bm \theta)+ \langle \nabla_{\bm x}C(\bm x(W-1); \bm \theta-\bm \delta(1)), \bm x(W)-\bm x(W-1)\rangle\\
	& +L \|\bm x(W)-\bm x(W-1)\|^2 + \frac{h^2}{2L}\|\bm \delta(1)\|^2\\
	= \ & \min_{\bm x \in \X^T}\left\{\langle \nabla_{\bm x}C(\bm x(W-1); \bm \theta-\bm \delta(1)), \bm x-\bm x(W-1)\rangle+L \|\bm x-\bm x(W-1)\|^2  \right\}\\
	& +C(\bm x(W-1); \bm \theta)+ \frac{h^2}{2L}\|\bm \delta(1)\|^2\\
	\leq \ & \langle \nabla_{\bm x}C(\bm x(W-1); \bm \theta-\bm \delta(1)), \bm x^*-\bm x(W-1)\rangle+L \|\bm x^*-\bm x(W-1)\|^2 \\
	& +C(\bm x(W-1); \bm \theta)+ \frac{h^2}{2L}\|\bm \delta(1)\|^2\\
	\leq \ & C(\bm x^*;\bm \theta)+ (L-\frac{\alpha}{4})r_{W-1}+\left(\frac{h^2}{\alpha}+  \frac{h^2}{2L}\right)\|\bm \delta(1)\|^2
		\end{align*}
		where we used Lemma \ref{lem: smooth with error C} in the first inequality, Lemma \ref{lem: project GD= min} and $\eta=\frac{1}{2L}$ in the first equality, Lemma \ref{lem: strong convex with error C} in the last inequality.
		By rearranging terms, we obtain
		\begin{equation}\label{equ: bdd Reg by rW-1}
	\text{Reg}(RHIG)= C(\bm x(W); \bm \theta)-C(\bm x^*; \bm \theta)\leq L\rho r_{W-1}+ \zeta \|\bm \delta(1)\|^2
		\end{equation}
			where $\rho=1-\frac{\alpha}{4L}$, $\zeta=\frac{h^2}{\alpha}+\frac{h^2}{2L}$.

\nit{Step 2: a recursive inequality between $r_{k+1}$ and $r_{k}$.}\\
In the following, we will show that
\begin{equation}
r_{k+1}\leq \rho r_k + \frac{\zeta}{L} \|\bm \delta(W-k)\|^2, \quad \forall\ 0 \leq k \leq W-1 \label{equ: recursive rk+1<rho rk}
\end{equation} 

Firstly, by \eqref{equ: inexact gd appendix}, $\eta=\frac{1}{2L}$, Lemma \ref{lem: project GD= min} and its first-order optimality condition, we have 
$$ \langle \nabla_{\bm x} C(\bm x(k); \bm \theta-\delta(W-k))+ 2L(\bm x(k+1)-\bm x(k)), \bm x-\bm x(k+1)\rangle \geq 0, \quad \forall \ \bm x\in \X^T$$
By substituting $\bm x=\bm x^*$ and rearranging terms, we obtain
\begin{equation}\label{equ: first order cond PGD}
\frac{1}{2L}\langle  \nabla_{\bm x} C(\bm x(k); \bm \theta-\delta(W-k)), \bm x^*-\bm x(k+1)\rangle \geq \langle \bm x(k+1)-\bm x(k), \bm x(k+1)-\bm x^*\rangle
\end{equation} 
Next, we will derive the recursive inequality \eqref{equ: recursive rk+1<rho rk} by using \eqref{equ: first order cond PGD}.
\begin{align*}
r_{k+1}= \ & \| \bm x(k+1)-\bm x^*\|^2 = \| \bm x(k+1)-\bm x(k)+ \bm x(k)-\bm x^*\|^2\\
= \ & r_k - \| \bm x(k+1)-\bm x(k)\|^2 + 2 \langle \bm x(k+1)-\bm x(k), \bm x(k+1)-\bm x^*\rangle \\
\leq \ & r_k - \| \bm x(k+1)-\bm x(k)\|^2 + \frac{1}{L}\langle  \nabla_{\bm x} C(\bm x(k); \bm \theta-\bm \delta(W-k)), \bm x^*-\bm x(k+1)\rangle\\
= \ & r_k - \| \bm x(k+1)-\bm x(k)\|^2 + \frac{1}{L}\langle  \nabla_{\bm x} C(\bm x(k); \bm \theta-\bm \delta(W-k)), \bm x^*-\bm x(k)\rangle\\
& +  \frac{1}{L}\langle  \nabla_{\bm x} C(\bm x(k); \bm \theta-\bm \delta(W-k)), \bm x(k)-\bm x(k+1)\rangle\\
= \ & r_k+  \frac{1}{L}\langle  \nabla_{\bm x} C(\bm x(k); \bm \theta-\bm \delta(W-k)), \bm x^*-\bm x(k)\rangle\\
&- \frac{1}{L}\left( \langle  \nabla_{\bm x} C(\bm x(k); \bm \theta-\bm \delta(W-k)), \bm x(k+1)-\bm x(k)\rangle + L \| \bm x(k+1)-\bm x(k)\|^2 \right)\\
\leq \ & r_k+  \frac{1}{L}\langle  \nabla_{\bm x} C(\bm x(k); \bm \theta-\bm \delta(W-k)), \bm x^*-\bm x(k)\rangle\\
& - \frac{1}{L}\left(C(\bm x(k+1);\bm \theta)-C(\bm x(k);\bm \theta)-\frac{h^2}{2L}\| \bm \delta(W-k)\|^2\right)\\
\leq  \ & r_k+  \frac{1}{L}\langle  \nabla_{\bm x} C(\bm x(k); \bm \theta-\bm \delta(W-k)), \bm x^*-\bm x(k)\rangle \\
&-\frac{1}{L}\left(C(\bm x^*;\bm \theta)-C(\bm x(k);\bm \theta)\right)+\frac{h^2}{2L^2}\|\bm \delta(W-k)\|^2\\
= \ & r_k -\frac{1}{L}\left(C(\bm x^*;\bm \theta)-C(\bm x(k);\bm \theta)+\langle  \nabla_{\bm x} C(\bm x(k); \bm \theta-\bm ), \bm x(k)-\bm x^*\rangle\right)\\
&+\frac{h^2}{2L^2}\|\bm \delta(W-k)\|^2\\
\leq \ & r_k -\frac{1}{L}\left( \frac{\alpha}{4}\| \bm x(k)-\bm x^*\|^2-\frac{h^2}{\alpha} \|\bm \delta(W-k)\|^2\right)+\frac{h^2}{2L^2}\|\bm \delta(W-k)\|^2\\
= \ & \rho r_k + \frac{\zeta}{L} \|\bm \delta(W-k)\|^2
\end{align*}
which completes the proof of \eqref{equ: recursive rk+1<rho rk}.

\nit{Step 3: completing the proof by \eqref{equ: recursive rk+1<rho rk} and \eqref{equ: bdd Reg by rW-1}.}\\
By summing \eqref{equ: recursive rk+1<rho rk} over $k=0, \dots, W-2$, we obtain
\begin{align*}
r_{W-1}&\leq \rho^{W-1} r_0 + \frac{\zeta}{L}  \left(\|\bm \delta(2)\|^2 + \rho \|\bm \delta(3)\|^2 + \dots + \rho^{W-2}\|\bm \delta(W)\|^2\right)\\
& \leq  \rho^{W-1} \frac{2}{\alpha} (C(\bm x(0);\bm \theta)-C(\bm x^*; \bm \theta))+ \frac{\zeta}{L} \sum_{k=2}^W \rho^{k-2}\|\bm\delta(k)\|^2
\end{align*}
By \eqref{equ: bdd Reg by rW-1}, we obtain the regret bound in Theorem \ref{thm: general}:
\begin{align*}
\text{Reg}(RHIG)&\leq L\rho ( \rho^{W-1} \frac{2}{\alpha} (C(\bm x(0);\bm \theta)-C(\bm x^*; \bm \theta))+ \frac{\zeta}{L}  \sum_{k=2}^W \rho^{k-2}\|\bm\delta(k)\|^2)+ \zeta \|\bm \delta(1)\|^2\\
& = \frac{2L}{\alpha}\rho^W\text{Reg}(\phi) +  \zeta  \sum_{k=1}^W \rho^{k-1}\|\bm\delta(k)\|^2\\
& =  \frac{2L}{\alpha}\rho^W\text{Reg}(\phi) +  \zeta  \sum_{k=1}^{\min(W,T)} \rho^{k-1}\|\bm\delta(k)\|^2+  \zeta  \one_{(W>T)}\sum_{k=T+1}^W \rho^{k-1}\|\bm\delta(T)\|^2\\
& =  \frac{2L}{\alpha}\rho^W\text{Reg}(\phi) +  \zeta  \sum_{k=1}^{\min(W,T)} \rho^{k-1}\|\bm\delta(k)\|^2+\zeta  \one_{(W>T)}\frac{\rho^T-\rho^{W}}{1-\rho}\|\bm\delta(T)\|^2
\end{align*}
where we used the fact that $\|\bm \delta(k)\|=\|\bm \delta(T)\|$ when $k>T$.

\section{Proofs of Theorem \ref{thm: ogd dyn regret}, Corollary \ref{cor: RHIG-OGD determine} and the claimed properties of the regret bound in Section \ref{sec: regret general}}\label{append: ogd}

In this section, we provide a dynamic regret bound for the restarted OGD initialization rule in Section \ref{sec: regret general}, based on which we prove Corollary \ref{cor: RHIG-OGD determine}. To achieve this, we will first establish a static regret bound for OGD initialization \eqref{equ: OGD initialization}. The proof is inspired by \cite{besbes2015non}.

For notational simplicity,  we slightly abuse the notation and let $x_t$ denote $x_t(0)$ generated by OGD. Further, by the definition of the prediction errors $\delta_{t-1}(W)$ for $W\geq 1$, we can write the initialization rule \eqref{equ: OGD initialization} as the following, which can be interpreted as  OGD with inexact gradients:
\begin{align}\label{equ: restarted ogd appendix}
x_t=\Pi_{\X}[x_{t-1}-\xi_t \nabla_{x} f(x_{t-1}; \theta_{t-1}-\delta_{t-1}(\min(W,T)))], \quad t\geq 2;
\end{align}
and $x_1=x_0$.
Here, we used the facts that $\theta_{t-1\mid t-W-1}=\theta_{t-1}-\delta_{t-1}(W)$ and $\delta_{t-1}(W)=\delta_{t-1}(T)$ for $W>T$.

\subsection{Static regret bound for  OGD with inexact gradients}
In this section, we consider the  OGD with inexact gradients \eqref{equ: restarted ogd appendix} with diminishing stepsize ${\xi_t= \frac{4}{\alpha t}}$ for $t\geq 1$. We will prove its static regret bound below.
\begin{theorem}[Static regret of OGD with inexact gradients]\label{thm: OGD static regret}
Consider the  OGD with inexact gradients \eqref{equ: restarted ogd appendix} with diminishing stepsize ${\xi_t= \frac{4}{\alpha t}}$ for $t\geq 1$ and any $x_0$. Then, for  $z^*=\argmin_{z \in \X}\sum_{t=1}^T f(z;\theta_t) $, we have the following static regret bound:
	\begin{align*}
	\sum_{t=1}^T[f(x_t;\theta_t)-f(z^*;\theta_t)]\leq \frac{2G^2}{\alpha}\log(T+1)+ \sum_{t=1}^T\frac{h^2}{\alpha}\| \delta_t(\min(W,T))\|^2
	\end{align*}
	Further, the total switching cost can be bounded by:
	$$ \sum_{t=1}^T d(x_t, x_{t-1}) \leq \frac{16G^2\beta}{\alpha^2}$$
\end{theorem}

\begin{proof}
	Firstly, we prove the static regret bound. Define $q_t=\|x_t-z^*\|^2$. Then, for $t\geq 1$, we have the following.
	\begin{align*}
	q_{t+1}=&\|x_{t+1}-z^*\|^2\leq \| x_t- \xi_{t+1} \nabla_x f(x_t;\theta_t-\delta_t(\min(W,T)))-z^*\|^2\\
	=& q_t + \xi_{t+1}^2\| \nabla_x f(x_t;\theta_t-\delta_t(\min(W,T)))\|^2-2\xi_{t+1} \langle x_t-z^*, \nabla_x f(x_t;\theta_t-\delta_t(\min(W,T)))\rangle\\
	\leq &q_t + \xi_{t+1}^2 G^2\! -\!2\xi_{t+1} \langle x_t-z^*,\! \nabla_x f(x_t;\theta_t)\rangle\\
	&-2\xi_{t+1} \langle x_t-z^*, \nabla_x f(x_t;\theta_t-\delta_t(\min(W,T)))
	-\nabla_x f(x_t;\theta_t)\rangle
	\end{align*}
	where the last inequality uses Assumption \ref{ass: d <beta (x-y)^2/2}(i).
	By rearranging terms, we obtain
	\begin{align}\label{equ: <xt-z*,d f>}
 \langle x_t-z^*, \nabla_x f(x_t;\theta_t)\rangle \leq \frac{q_t-q_{t+1}}{2\xi_{t+1}}+ \frac{\xi_{t+1}}{2}G^2- \langle x_t-z^*, \nabla_x f(x_t;\theta_t-\delta_t(\min(W,T)))-\nabla_x f(x_t;\theta_t)\rangle
	\end{align}
	
	By the strong convexity of $f(x;\theta_t)$, we have $f(z^*;\theta_t)\geq f(x_t;\theta_t)+\langle z^*-x_t, \nabla_x f(x_t;\theta_t)\rangle + \frac{\alpha}{2}\|z^*-x_t\|^2$. By rearranging terms and by \eqref{equ: <xt-z*,d f>}, we obtain
	\begin{align*}
&f(x_t;\theta_t)-	f(z^*;\theta_t)\leq \langle x_t-z^*, \nabla_x f(x_t;\theta_t)\rangle - \frac{\alpha}{2}\|z^*-x_t\|^2\\
\leq & \frac{q_t-q_{t+1}}{2\xi_{t+1}}+ \frac{\xi_{t+1}}{2}G^2- \langle x_t-z^*, \nabla_x f(x_t;\theta_t-\delta_t(\min(W,T)))-\nabla_x f(x_t;\theta_t)\rangle- \frac{\alpha}{2}\|z^*-x_t\|^2\\
 \leq& \frac{q_t-q_{t+1}}{2\xi_{t+1}}+ \frac{\xi_{t+1}}{2}G^2 + \| x_t-z^*\|\|\nabla_x f(x_t;\theta_t-\delta_t(\min(W,T)))-\nabla_x f(x_t;\theta_t)\| - \frac{\alpha}{2}\|z^*-x_t\|^2\\
 \leq & \frac{q_t-q_{t+1}}{2\xi_{t+1}}+ \frac{\xi_{t+1}}{2}G^2 +\frac{1}{\alpha} \|\nabla_x f(x_t;\theta_t-\delta_t(\min(W,T)))-\nabla_x f(x_t;\theta_t)\|^2-\frac{\alpha}{4}\|z^*-x_t\|^2\\
\leq &\frac{q_t-q_{t+1}}{2\xi_{t+1}}+ \frac{\xi_{t+1}}{2}G^2 + \frac{h^2}{\alpha}\|\delta_t(\min(W,T))\|^2 -\frac{\alpha}{4}q_t
	\end{align*}
	where we used $ab \leq \frac{\epsilon}{2}a^2+ \frac{1}{2\epsilon} b^2$ for any $a, b\in \R$ and any $\epsilon>0$ in the second last inequality and Assumption \ref{ass: grad lip cont} in the last inequality. 
	By summing over $t=1,\dots, T$, we obtain
	\begin{align*}
	\sum_{t=1}^T[f(x_t;\theta_t)-	f(z^*;\theta_t)]\leq \ &\sum_{t=2}^T\left(\frac{1}{2\xi_{t+1}}-\frac{1}{2\xi_t}-\frac{\alpha}{4}\right)q_t + \left(\frac{1}{2\xi_2}-\frac{\alpha}{4}\right)q_1-\frac{1}{\xi_{T+1}}q_{T+1} \\
	&+ \sum_{t=1}^T \frac{\xi_{t+1}}{2}G^2+ \sum_{t=1}^T\frac{h^2}{\alpha}\| \delta_t(\min(W,T))\|^2\\
	 \leq\ &\log(T+1) \frac{2G^2}{\alpha}+ \sum_{t=1}^T\frac{h^2}{\alpha}\| \delta_t(\min(W,T))\|^2
		\end{align*}
		which completes the proof of the static regret bound.
		
		Next, we bound the switching costs. By Assumption \ref{ass: d <beta (x-y)^2/2}(ii), we have
		\begin{align*}
		\sum_{t=1}^T d(x_t, x_{t-1})&\leq \sum_{t=1}^T \frac{\beta}{2}\|x_t-x_{t-1}\|^2 \\
		&\leq \sum_{t=1}^T \frac{\beta}{2} \| \xi_t \nabla_{x} f(x_{t-1};\theta_{t-1}-\delta_{t-1}(\min(W,T)))\|^2\\
		&\leq \frac{\beta G^2}{2}\sum_{t=1}^T \xi_t^2\leq \frac{16\beta G^2}{\alpha^2}
		\end{align*}

\end{proof}

\subsection{Proof of Theorem \ref{thm: ogd dyn regret}: dynamic regret bound for restarted OGD with inexact gradients}

	We denote the set of stages in epoch $k$ as $\mathcal T_k=\{k\Delta+1,\dots, \min(k\Delta+\Delta,T)\}$ for $k=0, \dots, \ceil{T/\Delta}-1$. We introduce $z_k^*= \argmin_{z\in \X}\sum_{t\in \mathcal T_k}[f(z;\theta_t)]$ for all $k$; $y_t^*=\argmin_{x_t\in \X}f(x_t;\theta_t)$ for all $t$; and $\bm x^*=\argmin_{\bm x\in \X^T} \sum_{t=1}^T[f(x_t;\theta_t)+ d(x_t, x_{t-1})]$. The dynamic regret of the restarted OGD with inexact gradients can be bounded as follows.
	\begin{align*}
	\text{Reg}(OGD)& = \sum_{t=1}^T [f(x_t;\theta_t)+ d(x_t, x_{t-1})]- \sum_{t=1}^T [f(x^*_t;\theta_t)+ d(x^*_t, x^*_{t-1})]\\
	& \leq  \sum_{t=1}^T [f(x_t;\theta_t)+ d(x_t, x_{t-1})]- \sum_{t=1}^T [f(x^*_t;\theta_t)]\\
	& \leq \sum_{t=1}^T [f(x_t;\theta_t)+ d(x_t, x_{t-1})]- \sum_{t=1}^T [f(y^*_t;\theta_t)]\\
	& = \sum_{k=0}^{\ceil{T/\Delta}-1}\sum_{t\in \mathcal T_k} [f(x_t;\theta_t)+ d(x_t, x_{t-1})- f(y^*_t;\theta_t)]\\
	& =   \sum_{k=0}^{\ceil{T/\Delta}-1}\sum_{t\in \mathcal T_k} [f(x_t;\theta_t)- f(z_k^*;\theta_t)]+\sum_{k=0}^{\ceil{T/\Delta}-1}\sum_{t\in \mathcal T_k} d(x_t, x_{t-1})\\
	& \quad \quad  + \sum_{k=0}^{\ceil{T/\Delta}-1}\sum_{t\in \mathcal T_k} [  f(z_k^*;\theta_t)-f(y^*_t;\theta_t)]\\
	& \leq  \ceil{T/\Delta}\log(\Delta+1)\frac{2G^2}{\alpha}+\frac{h^2}{\alpha} \|\bm\delta(\min(W,T))\|^2+ \ceil{T/\Delta} \frac{16\beta G^2}{\alpha^2}\\
	& \quad \quad +\sum_{k=0}^{\ceil{T/\Delta}-1}\sum_{t\in \mathcal T_k} [  f(z_k^*;\theta_t)-f(y^*_t;\theta_t)]
	\end{align*}
	where the first inequality uses Assumption \ref{ass: d <beta (x-y)^2/2}, the second inequality uses the optimality of $y_t^*$, the last inequality uses Theorem \ref{thm: OGD static regret} and the fact that the OGD considered here restarts at the beginning of each epoch $k$ and repeats the stepsizes defined in Theorem \ref{thm: OGD static regret}, thus satisfying the static regret bound and the switching cost bound in Theorem \ref{thm: OGD static regret} within each epoch. 
	
	Now, it suffices to bound $\sum_{k=0}^{\ceil{T/\Delta}-1}\sum_{t\in \mathcal T_k} [  f(z_k^*;\theta_t)-f(y^*_t;\theta_t)]$. By the optimality of $z_k^*$, we have:
	\begin{equation}\label{equ: star}
	    \sum_{k=0}^{\ceil{T/\Delta}-1}\sum_{t\in \mathcal T_k} [  f(z_k^*;\theta_t)-f(y^*_t;\theta_t)]\leq  \sum_{k=0}^{\ceil{T/\Delta}-1}\sum_{t\in \mathcal T_k} [  f(y^*_{k\Delta+1};\theta_t)-f(y^*_t;\theta_t)].
	\end{equation}
	We define $V^k = \sum_{t\in \mathcal T_k} \sup_{x\in \X} | f(x;\theta_t)-f(x;\theta_{t-1})|$. Then, for any $t\in \mathcal T_k$,  we obtain
	\begin{align*}
	f(y^*_{k\Delta+1};\theta_t)-f(y^*_t;\theta_t) = \ &	f(y^*_{k\Delta+1};\theta_t)-f(y^*_{k\Delta+1};\theta_{k\Delta+1}) + f(y^*_{k\Delta+1};\theta_{k\Delta+1})-f(y_t^*;\theta_{k\Delta+1})\\
	&+ f(y_t^*;\theta_{k\Delta+1})-f(y^*_t;\theta_t)\\
	\leq \ & V^k+0+V^k=2V^k
	\end{align*}
	By summing over $t\in \mathcal T_k$ and $k=0, \dots, \ceil{T/\Delta}-1$ and by the inequality \eqref{equ: star}, we obtain
	$$ \sum_{k=0}^{\ceil{T/\Delta}-1}\sum_{t\in \mathcal T_k} [  f(z_k^*;\theta_t)-f(y^*_t;\theta_t)]\leq \sum_{k=0}^{\ceil{T/\Delta}-1}\sum_{t\in \mathcal T_k} 2V^k= \sum_{k=0}^{\ceil{T/\Delta}-1} 2\Delta V^k= 2\Delta V_T$$
	Combining the bounds above yields the desired bound on the dynamic regret of OGD below by letting $\Delta=\ceil{\sqrt{2T/V_T}}$:
	\begin{align*}
	 \text{Reg}(OGD) &\leq \ceil{T/\Delta}\log(\Delta+1)\frac{2G^2}{\alpha}+\frac{h^2}{\alpha} \|\bm\delta(\min(W,T))\|^2+ \ceil{T/\Delta} \frac{16\beta G^2}{\alpha^2} +2\Delta V_T \\
	 & \leq \left(\sqrt{\frac{V_T T }{2} }+1\right)\log(2+\sqrt{2T/V_T})\frac{2G^2}{\alpha}+ \frac{h^2}{\alpha} \|\bm\delta(\min(W,T))\|^2+\left(\sqrt{\frac{V_T T}{2} }+1\right) \frac{16\beta G^2}{\alpha^2}\\
	 &\quad \ + 2(\sqrt{2V_T T}+V_T)\\
	 & \leq (\sqrt{V_T T /2 }+1)\log(2+\sqrt{2T/V_T}) \left( \frac{2G^2}{\alpha}+ \frac{16\beta G^2}{\alpha^2}+ 2(2+\sqrt 2)\right)+ \frac{h^2}{\alpha} \|\bm\delta(\min(W,T))\|^2\\
	 & \leq \sqrt{2V_T T} \log(2+\sqrt{2T/V_T})\left( \frac{2G^2}{\alpha}+ \frac{16\beta G^2}{\alpha^2}+ 2(2+\sqrt 2)\right)+ \frac{h^2}{\alpha} \|\bm\delta(\min(W,T))\|^2\\
	  & \leq \sqrt{V_T T} \log(1+\sqrt{T/V_T})\left( \frac{4\sqrt 2 G^2}{\alpha}+ \frac{32\sqrt 2\beta G^2}{\alpha^2}+ 8(1+\sqrt 2)\right)+ \frac{h^2}{\alpha} \|\bm\delta(\min(W,T))\|^2\\
	  &\leq \sqrt{V_T T} \log(1+\sqrt{T/V_T})\left( \frac{4\sqrt 2G^2}{\alpha}+ \frac{32\sqrt 2\beta G^2}{\alpha^2}+ 20\right)+ \frac{h^2}{\alpha} \|\bm\delta(\min(W,T))\|^2
\end{align*}
where we used the facts that $\ceil{x}\leq x+1$, $1\leq V_T\leq T$, $T>2$, $\log(2+\sqrt{2T/V_T})\leq 2\log(1+\sqrt{T/V_T})$, and $8(1+\sqrt 2)<20$.
	
	
%
%
%
%

\subsection{Proof of Corollary \ref{cor: RHIG-OGD determine}}
The proof is straightforward by substituting restarted OGD's regret bound in Theorem \ref{thm: ogd dyn regret} into the general regret bound in Theorem \ref{thm: general}, that is,
\begin{align*}
\textup{Reg}(RHIG)\leq \,&{\rho^W \frac{2L}{\alpha} C_1\sqrt{V_T T}\log(1+\sqrt{T/V_T})}\\
&+\! \frac{2L}{\alpha}\frac{h^2}{\alpha}\rho^{W}\!\|\bm\delta(\min(W,T))\|^2\!+\! \! \sum_{k=1}^{\min(W,T)}\! \!\zeta\rho^{k-1}\|\bm \delta(k)\|^2\!+\!\one_{(W>T)} \frac{\rho^T\!-\!\rho^W}{1-\rho}\zeta \|\bm \delta(T)\|^2.
\end{align*}

\subsection{Proofs of the monotonicity claims in the discussion of Corollary \ref{cor: RHIG-OGD determine}.}

In Section \ref{sec: regret general}, when discussing \nbf{Choices of $W$}, we claim that ``Part I increases with $V_T$ and Part II increases with the prediction errors. Further, as $W$ increases, Part I decreases  but  Part II increases.'' For completeness, we prove this claim below.

\paragraph{Properties of Part I $\rho^W \frac{2L}{\alpha} C_1\sqrt{V_T T}\log(1+\sqrt{T/V_T})$:}  Since $0<\rho <1$, it is straightforward that Part I monotonically decreases with $W$. Next, consider function $p(x)= x\log(1+ \frac{b}{x})$ for $x, b>0$. Since $p'(x)=\frac{x}{x+b}- 1 - \log(\frac{x}{x+b})\geq 0$ by $y-1 \geq \log(y)$ for any $y>0$, function $p(x)$ monotonically increases with $x$. Therefore, for any fixed $W$, Part I monotonically increases with  $\sqrt{V_T}$ and thus $V_T$.

\paragraph{Properties of Part II $\frac{2L h^2}{\alpha^2}\rho^{W}\!\|\bm\delta(\min(W,T))\|^2\!+\!\!  \sum_{k=1}^{\min(W,T)}\! \!\zeta\rho^{k-1}\|\bm \delta(k)\|^2\!+\!\one_{(W>T)} \frac{\rho^T\!-\!\rho^W}{1-\rho}\zeta \|\bm \delta(T)\|^2$:}  It is straightforward that Part II   monotonically increases with $\{\|\bm \delta(k)\|^2\}_{k=1}^W$. Next, we discuss the monotonicty with respect to $W$. We first consider $W\leq T$. In this case, Part II is equal to $\text{Part II}(W)\coloneqq \frac{2Lh^2}{\alpha^2}\rho^{W}\!\|\bm\delta(W)\|^2\!+\!\sum_{k=1}^{W}\!\zeta\rho^{k-1}\|\bm \delta(k)\|^2$. Notice that
\begin{align*}
\text{Part II}(W)-\text{Part II}(W-1)&= \frac{2Lh^2}{\alpha^2} \rho^W \|\bm \delta(W)\|^2 + \zeta \rho^{W-1} \|\bm \delta(W)\|^2-\frac{2Lh^2}{\alpha^2} \rho^{W-1} \|\bm \delta(W-1)\|^2\\
& \geq \left( \frac{2Lh^2}{\alpha^2} \rho+\zeta -\frac{2Lh^2}{\alpha^2} \right)\rho^{W-1} \|\bm \delta(W-1)\|^2\\
& = \left(\frac{h^2}{2\alpha}+\frac{h^2}{2L}\right)\rho^{W-1} \|\bm \delta(W-1)\|^2 > 0
\end{align*}
where we used $\|\bm \delta(W)\|^2 \geq \|\bm \delta(W-1)\|^2$, $\rho=1-\frac{\alpha}{4L}$, $\zeta=\frac{h^2}{\alpha}+\frac{h^2}{2L}$. Therefore, Part II is monotonically increasing with $W$ for $W\leq T$. Besides, we consider $W>T$. In this case, Part II is equal to  $\text{Part II}(W)\coloneqq \frac{2Lh^2}{\alpha^2}\rho^{W}\!\|\bm\delta(T))\|^2\!+\!\sum_{k=1}^{T}\!\zeta\rho^{k-1}\|\bm \delta(k)\|^2+ \frac{\rho^T-\rho^W}{1-\rho}\zeta \|\bm \delta(T)\|^2$. Notice that, when $W>T$, we have
\begin{align*}
\text{Part II}(W)-\text{Part II}(W-1)&= \frac{2Lh^2}{\alpha^2} (\rho^W-\rho^{W-1}) \|\bm \delta(T)\|^2 + \frac{\rho^{W-1}-\rho^W}{1-\rho}\zeta \|\bm \delta(T)\|^2\\
& = \left( \frac{2Lh^2}{\alpha^2} (\rho-1)+\zeta  \right)\rho^{W-1} \|\bm \delta(T)\|^2\\
& = \left(\frac{h^2}{2\alpha}+\frac{h^2}{2L}\right)\rho^{W-1} \|\bm \delta(W-1)\|^2 > 0
\end{align*}
In conclusion, Part II increases with $W$ for $W\geq 1$.

\section{Analysis on the special case in Section \ref{sec: regret general}}\label{append: special case}
\subsection{Proof of Corollary \ref{cor: special case}}

For notational simplicity, let $R(W)$ denote the regret bound in Corollary \ref{cor: RHIG-OGD determine} given lookahead horizon $W$, i.e.
\begin{align*}
R(W)= \,&{\rho^W \frac{2L}{\alpha} C_1\sqrt{V_T T}\log(1+\sqrt{T/V_T})}\\
&+\! \frac{2L}{\alpha}\frac{h^2}{\alpha}\rho^{W}\!\|\bm\delta(\min(W,T))\|^2\!+\! \! \sum_{k=1}^{\min(W,T)}\! \!\zeta\rho^{k-1}\|\bm \delta(k)\|^2\!+\!\one_{(W>T)} \frac{\rho^T\!-\!\rho^W}{1-\rho}\zeta \|\bm \delta(T)\|^2.
\end{align*}
We will show that $R(W)\leq R(W-1)$ for $W\geq 1$. Firstly, we consider $W\leq T$. In this case, we have $R(W)= \rho^W \frac{2L}{\alpha} C_1\sqrt{V_T T}\log(1+\sqrt{T/V_T})+\frac{2L}{\alpha}\frac{h^2}{\alpha}\rho^{W} \|\bm \delta(W)\|^2 + \sum_{k=1}^{W}\zeta\rho^{k-1}\|\bm \delta(k)\|^2$. Notice that
\begin{align*}
R(W)-R(W-1)=\ & ( \rho^W -\rho^{W-1})\frac{2L}{\alpha} C_1\sqrt{V_T T}\log(1+\sqrt{T/V_T})+\frac{2L}{\alpha}\frac{h^2}{\alpha}\rho^{W} \|\bm \delta(W)\|^2 \\
&+ \zeta\rho^{W-1}\|\bm \delta(W)\|^2-\frac{2L}{\alpha}\frac{h^2}{\alpha}\rho^{W-1} \|\bm \delta(W-1)\|^2\\
 \leq& \rho^{W-1}\left( (\frac{2L}{\alpha}\frac{h^2}{\alpha}\rho+\zeta )\|\bm \delta(W)\|^2-(1-\rho) \frac{2L}{\alpha} C_1\sqrt{V_T T}\log(1+\sqrt{T/V_T})\right)\\
 \leq& 0
\end{align*}
when the following condition holds for any $W\leq T$.
\begin{equation}\label{equ: condition appendix}
(\frac{2L}{\alpha}\frac{h^2}{\alpha}\rho+\zeta )\|\bm \delta(W)\|^2\leq (1-\rho) \frac{2L}{\alpha} C_1\sqrt{V_T T}\log(1+\sqrt{T/V_T})
\end{equation} 

Next, we consider $W>T$. In this case, we have $R(W)= \rho^W \frac{2L}{\alpha} C_1\sqrt{V_T T}\log(1+\sqrt{T/V_T})+ \frac{2L}{\alpha}\frac{h^2}{\alpha}\rho^{W}\|\bm \delta(T)\|^2 + \sum_{k=1}^{T}\! \!\zeta\rho^{k-1}\|\bm \delta(k)\|^2+\frac{\rho^T\!-\!\rho^W}{1-\rho}\zeta \|\bm \delta(T)\|^2.$ Therefore, 
\begin{align*}
R(W)-R(W-1) =& (\rho^W-\rho^{W-1}) \frac{2L}{\alpha} C_1\sqrt{V_T T}\log(1+\sqrt{T/V_T})+ \frac{2L}{\alpha}\frac{h^2}{\alpha}(\rho^{W}-\rho^{W-1})\|\bm \delta(T)\|^2 \\
& + \frac{\rho^{W-1}-\!\rho^W}{1-\rho}\zeta \|\bm \delta(T)\|^2\\
\leq  & \rho^{W-1} \left( \zeta \|\bm \delta(T)\|^2- (1-\rho ) \frac{2L}{\alpha} C_1\sqrt{V_T T}\log(1+\sqrt{T/V_T}) \right)\\
\leq & 0
\end{align*}
given the condition \eqref{equ: condition appendix}.

In conclusion, we have $R(W)\leq R(W-1)$ for $W\geq 1$ and the $R(W)$ is minimized  by letting $W\to +\infty$. Further, when $W\to +\infty$, we have the following bound. 
\begin{align*}
\lim_{W\to+\infty} R(W)&=\sum_{k=1}^{T}\zeta\rho^{k-1}\|\bm \delta(k)\|^2+\frac{\rho^T}{1-\rho}\zeta \|\bm \delta(T)\|^2  \leq \frac{\zeta}{1-\rho }\sum_{k=1}^{T}\rho^{k-1}\|\bm \delta(k)\|^2 
\end{align*}

\subsection{Proof of Theorem \ref{thm: lower bdd}}

Without loss of generality, we consider $n=1$. It is straightforward to generalize the proof to $n>1$ cases. The proof is based on constructing a special cost function where the lower bound holds.

Consider cost function $f(x_t;\theta_t)=\frac{\alpha}{2}(x_t^2-2\theta_t x_t )$ and $d(x_t,x_{t-1})=\frac{\beta}{2}\|x_t-x_{t-1}\|^2$ on $\X= [-1/2, 1/2]$, where  $l_f=\alpha$, $l_d=2\beta$, $L=\alpha+4\beta$ and $h=\alpha$. Let $\alpha>1$ and $\frac{\beta}{\alpha}<4+3\sqrt 2$ so that $\rho_0<1/2$. Let   $\theta_t \in \X$ for all $t$,  then we have $G=\sup_{x\in \X} \| \alpha(x-\theta_t)\|=\alpha$. Let $x_0=0$.

Consider a random $\theta_t$:
$$\theta_t=\mu_t +e_1^t+ \dots + e_t^t, \quad \forall \ 1\leq t \leq T,$$
where $e_{\tau}^t$ are independent variables across $1\leq t \leq T$ and $1\leq \tau \leq t$. Let the support of $e_{\tau}^t$ be $[-\frac{1}{8t}, \frac{1}{8t} ]$ and let $\mu_t=(-1)^t\frac{1}{4}$, so $\theta_t\in \X$ is $\X$ and $\frac{1}{8} \leq \theta_t \leq \frac{3}{8}$ if $t$ is even and $\frac{-3}{8} \leq \theta_t \leq \frac{-1}{8}$ if $t$ is odd. Consider predictions at time $\tau$ as $\theta_{t\mid \tau}=\mu_t+ e_1^t+ \dots + e_{\tau}^t$ for any $0\leq \tau \leq t$. Therefore, $\delta_t(t-\tau)= e_{\tau+1}^t + \dots + e_{t}^t$ and $\|\delta_t(t-\tau)\|\leq \frac{t-\tau}{8t}\leq 1/8$.
According to our construction, we have that $V_T= \sum_{t=1}^T \sup_{x\in \X} \alpha\|(\theta_t-\theta_{t-1})x\|\geq  \alpha T/8$; and $\|\bm \delta(k)\|^2 \leq \frac{T}{64}$ for any $k\geq 1$. Then, it is straightforward to verify  that the constructed cost functions and predictions satisfy $ \sqrt{V_T T}\log(1+\sqrt{T/V_T}) \geq \frac{{2L} h^2\rho+{\alpha^2} \zeta}{2L C_1 (1-\rho){\alpha}}\|\bm \delta(k)\|^2 $ for any $k\geq 1$.


 Notice that knowing $\theta_{t\mid 0}, \dots, \theta_{t\mid \tau}$ is equivalent with knowing $\mu_t, e_1^t, \dots, e_{\tau}^t$. Therefore, let filtration $\F_t$ denote all the information at $t$ provide by the predictions and the history, then $\F_t$ is generated by $\theta_1, \dots, \theta_{t-1}$ and $\mu_s, e_1^{s}, \dots, e_{t-1}^s$ for $s\geq t$. Notice that $\E[\theta_{\tau}\mid \F_t]=\theta_{\tau \mid t-1}$ for $\tau\geq t$. Besides, for any online algorithm $\A$, we have that $x_t^{\A}$ is measurable in $\F_t$.

Since $\theta_t \in \X$ for all $t$, it can be shown that the optimal solution $\bm x^*=\argmin_{\bm x \in \X^T} C(\bm x;\bm \theta)$  is an interior point of $\X^T$ and thus satisfies the first-order optimality condition $\bm x^*=A \bm \theta$, where $A$ is the inverse of the Hessian matrix of $C(\bm x;\bm \theta)$. Equivalently, we have $x_t^*=\sum_{\tau=1}^T a_{t,\tau} \theta_{\tau}$. Further, Lemma 5 in \cite{li2018online} shows that $a_{t,\tau}^2 \geq c_2\rho_0^{\tau-t}$ for $\tau\geq t$, where $c_2=(\frac{\alpha}{\alpha+\beta})^2(1-\sqrt{\rho_0})^2$. 

Since $x_t^{\A}$ is measurable in $\F_t$, by the projection theory, we have
\begin{align*}
\E[\|x_t^{\A}-x_t^*\|^2] &\geq \E[\|\E[x_t^*\mid \F_t]-x_t^*\|^2].
\end{align*}
Notice that \begin{align*}
\E[x_t^*\mid \F_t]&=a_{t,1}\theta_1+ \dots + a_{t,t-1}\theta_{t-1} + a_{t,t}\E[\theta_t\mid \F_t]+ \dots + a_{t,T}\E[\theta_T\mid \F_t]\\
& = a_{t,1}\theta_1+ \dots + a_{t,t-1}\theta_{t-1} + a_{t,t} \theta_{t\mid t-1} + a_{t,T} \theta_{T\mid t-1}
\end{align*}
Therefore, 
\begin{align*}
\E[\|\E[x_t^*\mid \F_t]-x_t^*\|^2]&= \E[\| a_{t,t}\delta_t(1)+\dots + a_{t,T}\delta_T(T-t+1)\|^2]\\
& = a_{t,t}^2 \E[\|\delta_t(1)\|^2]+ \dots + a_{t,T}^2 \E[\|\delta_{T}(T-t+1)\|^2]\\
& \geq c_2( \E[\|\delta_t(1)\|^2]+ \dots + \rho_0^{T-t}\E[\|\delta_{T}(T-t+1)\|^2])
\end{align*}
where we used the independence among the prediction errors and $a_{t,\tau}^2 \geq c_2\rho_0^{\tau-t}$ for $\tau \geq t$.

Summing over $t$ leads to the following.
\begin{align*}
\sum_{t=1}^T \E[\|x_t^{\A}-x_t^*\|^2] &\geq \sum_{t=1}^T  \E[\|\E[x_t^*\mid \F_t]-x_t^*\|^2]\\
& \geq \sum_{t=1}^T (a_{t,t}^2 \E[\|\delta_t(1)\|^2]+ \dots + a_{t,T}^2 \E[\|\delta_{T}(T-t+1)\|^2])\\
& \geq \sum_{t=1}^T c_2\sum_{k=1}^{T-t+1}  \rho_0^{k-1}\E[\|\delta_{k+t-1}(k)\|^2]\\
& =c_2 \sum_{k=1}^T  \rho_0^{k-1} \sum_{t=1}^{T+1-k}\E[\|\delta_{k+t-1}(k)\|^2]\\
& = c_2 \sum_{k=1}^T  \rho_0^{k-1} \sum_{t=1}^{T}\E[\|\delta_{t}(k)\|^2]- c_2 \sum_{k=1}^T  \rho_0^{k-1} \sum_{t=1}^{k-1}\E[\|\delta_{t}(k)\|^2]\\
& =c_2 \sum_{k=1}^T  \rho_0^{k-1} \sum_{t=1}^{T}\E[\|\delta_{t}(k)\|^2]- c_2 \sum_{k=1}^T  \rho_0^{k-1} \sum_{t=1}^{k-1}\E[\|\delta_{t}(t)\|^2]\\
& = c_2 \sum_{k=1}^T  \rho_0^{k-1} \sum_{t=1}^{T}\E[\|\delta_{t}(k)\|^2]-  c_2 \sum_{t=1}^{T-1} \E[\|\delta_{t}(t)\|^2]\sum_{k=t+1}^T \rho_0^{k-1}\\
& \geq c_2 \sum_{k=1}^T  \rho_0^{k-1} \sum_{t=1}^{T}\E[\|\delta_{t}(k)\|^2]-  c_2 \sum_{k=1}^{T-1} \E[\|\delta_{k}(k)\|^2] \frac{\rho_0^k}{1-\rho_0}\\
& \geq c_2 \sum_{k=1}^T  \rho_0^{k-1} \sum_{t=1}^{T}\E[\|\delta_{t}(k)\|^2]\frac{1-2\rho_0}{1-\rho_0}
\end{align*}
where we used  $\delta_t(k)=\delta_t(t)$ for $k\geq t$ in the third equality and change the counting index from $t$ to $k$ in the second last inequality. 

By strong convexity, we have $\E[\text{Reg}(\A)]\geq \frac{\alpha}{2}\E\|\bm x^{\A}-\bm x^*\|^2 \geq c_2\frac{\alpha}{2}\frac{1-2\rho_0}{1-\rho_0} \sum_{k=1}^T  \rho_0^{k-1} \E[\| \bm \delta(k)\|^2] $. Therefore, there must exist a scenario such that  $\text{Reg}(\A) \geq c_2\frac{\alpha}{2}\frac{1-2\rho_0}{1-\rho_0} \sum_{k=1}^T  \rho_0^{k-1} \| \bm \delta(k)\|^2 $. Since $h=\alpha$ in our construction, we  complete the proof.

%
%
%
%
%
%
%
%
%
%

\section{Stochastic Regret Analysis}\label{append: stochastic}

\subsection{Proof of Theorem \ref{thm: expected regret bounds}}

By taking expectation on both sides of the regret bound in Theorem \ref{thm: general}, we have
\begin{equation}\label{equ: expected general proof}
\begin{aligned}
\E[\text{Reg}(RHIG)]\leq \frac{2L}{\alpha}\rho^W \E[\text{Reg}(\phi)]+ \zeta \sum_{k=1}^{\min(W,T)}\rho^{k-1}\E[\|\bm \delta(k)\|^2]+ \one_{(W>T)} \frac{\rho^T-\rho^W}{1-\rho}\zeta\E[ \|\bm \delta(T)\|^2],
\end{aligned}
\end{equation}
Therefore, it suffices to bound $\E[\|\bm \delta(k)\|^2]$ for $1\leq k \leq T$. By $\bm \delta(k)=(\delta_1(k)^\top, \dots, \delta_T(k)^\top)^\top$, $\delta_t(k)= \theta_t-\theta_{t\mid t-k}=P(0)e_t+ \dots + P(k-1)e_{t-k+1}$ for $k\leq t$ and $\delta_t(k)=\delta_t(t)$ for $k>t$, we have
\begin{align}\label{equ: delta(k)=M e}
\bm \delta(k)= \bm M_k \bm e,\quad 1\leq k \leq T
\end{align}
where we define $\bm e=(e_1^\top, \dots, e_T^\top)^\top\in \R^{qT}$ and
{\small$$ \bm M_k=\begin{bmatrix}
	P(0) & 0 & \dots & \dots & \dots  & 0\\
	P(1) & P(0) &  \dots & \dots &\dots   & 0\\
	\vdots & \ddots & \ddots &\ddots  &\ddots   & 0\\
	P(k-1) & \dots & P(1) & P(0) &\dots  & 0\\
	\vdots & \ddots &  & \ddots & \ddots   & 0\\
	0 & \dots   & P(k-1) & \dots & P(1) & P(0)
	\end{bmatrix}.$$}

Let $\bm{R_e}$ denote the covariance matrix of $\bm e$, i.e. 
$$ \bm{R_e}=\begin{bmatrix}
R_e & 0 & \dots & 0\\
0 & R_e & \dots & 0\\
\vdots & \ddots & \ddots & \vdots\\
0 & \dots & \dots & R_e
\end{bmatrix}$$

Then, for $k\leq T$, we have
\begin{align}
\E[\|\bm \delta(k)\|^2]& = \E[\bm e^\top \bm M_k^\top \bm M_k \bm e] = \E[\text{tr}(\bm e \bm e^\top  \bm M_k^\top \bm M_k)]\notag\\
& = \text{tr}\left(\bm{R_e} \bm M_k^\top \bm M_k \right)\notag\\
& \leq  \|R_e\|_2 \|\bm M_k\|_F^2\notag= \|R_e\|_2 \sum_{t=0}^{k-1}(T-t)\|P(t)\|_F^2 \notag
\end{align}
where the first inequality is by $\text{tr}(AB)\leq \|A\|_2 \text{tr}(B)$ for any symmetrix matrices $A, B$, and  $\| \text{diag}(R_e, \dots, R_e)\|_2=\|R_e\|_2$ and $\text{tr}(A^\top A)=\|A\|_F^2$ for any matrix $A$. In addition, for $k\geq T$, we have 
$ \E[\|\bm \delta(k)\|^2]\leq \|R_e\|_2 \sum_{t=0}^{T-1}(T-t)\|P(t)\|_F^2$. In conclusion, for any $k\geq 1$, we have
\begin{align}\label{equ: bound on E(delta(k))}
\E[\|\bm \delta(k)\|^2]\leq \|R_e\|_2 \sum_{t=0}^{\min(k,T)-1}(T-t)\|P(t)\|_F^2
\end{align}

When $W\leq T$, substituting the bounds on $\E[\|\bm \delta(k)\|^2]$  into \eqref{equ: expected general proof}  yields the  bound on the expected regret below. 
\begin{align*}
\E[\text{Reg}(RHIG)]&\leq  \frac{2L}{\alpha}\rho^W \E[\text{Reg}(\phi)]+\zeta\sum_{k=1}^{W}\rho^{k-1}\|R_e\|_2 \sum_{t=0}^{k-1}(T-t)\|P(t)\|_F^2\\
& =  \frac{2L}{\alpha}\rho^W \E[\text{Reg}(\phi)]+\zeta\sum_{t=0}^{W-1} \|R_e\|_2 (T-t)\|P(t)\|_F^2 \sum_{k=t+1}^W\rho^{k-1}\\
&= \frac{2L}{\alpha}\rho^W \E[\text{Reg}(\phi)]+\zeta\sum_{t=0}^{W-1} \|R_e\|_2 (T-t)\|P(t)\|_F^2 \frac{\rho^t-\rho^W}{1-\rho}
\end{align*}
When $W\geq T$, substituting the bounds on $\E[\|\bm \delta(k)\|^2]$  into \eqref{equ: expected general proof}  yields the  bound on the expected regret below. 
\begin{align*}
\E[\text{Reg}(RHIG)]\leq\ &  \frac{2L}{\alpha}\rho^W \E[\text{Reg}(\phi)]+\zeta\sum_{k=1}^{T}\rho^{k-1}\|R_e\|_2 \sum_{t=0}^{k-1}(T-t)\|P(t)\|_F^2\\
&+ \zeta \frac{\rho^T-\rho^W}{1-\rho} \|R_e\|_2 \sum_{t=0}^{T-1}(T-t)\|P(t)\|_F^2\\
 = \ & \frac{2L}{\alpha}\rho^W \E[\text{Reg}(\phi)]+\zeta\sum_{t=0}^{T-1} \|R_e\|_2 (T-t)\|P(t)\|_F^2 (\sum_{k=t+1}^T\rho^{k-1} + \frac{\rho^T-\rho^W}{1-\rho})\\
=\ & \frac{2L}{\alpha}\rho^W \E[\text{Reg}(\phi)]+\zeta\sum_{t=0}^{T-1} \|R_e\|_2 (T-t)\|P(t)\|_F^2 \frac{\rho^t-\rho^W}{1-\rho}
\end{align*}
In conclusion, we have the regret bound for general $W\geq 0$ below.
\begin{align*}
\E[\text{Reg}(RHIG)]&\leq \frac{2L}{\alpha} \rho^W \text{Reg}(\phi)+\zeta\sum_{t=0}^{\min(W,T)-1} \|R_e\|_2 (T-t)\|P(t)\|_F^2 \frac{\rho^t-\rho^W}{1-\rho}
\end{align*}
%

\subsection{Proof of Corollary \ref{cor: RHIG-OGD-stoch}}

Before the proof, we note that we cannot apply the expected regret bound in \cite{besbes2015non} directly due to the  major differences in the problem formulation as discussed below. Firstly,  the expected regret definition considered in this paper is different from that in \cite{besbes2015non} because the true cost function parameter $\theta_t$ in our case is also random and taken expectation on, while the true cost function in \cite{besbes2015non} is deterministic and the expectation is only taken on the random gradient noises. Besides, \cite{besbes2015non} considers unbiased gradient estimation while our gradient estimation $\nabla_{x_t}f(x_t;\theta_{t\mid \tau})$ can be biased. Further, \cite{besbes2015non} considers independent gradient noises at each stage $t$, while our gradient noises are correlated due to the correlation among prediction errors. Therefore, we have to revise the original proof in \cite{besbes2015non} for a new regret bound for our setting.

Similar to the proof of Theorem \ref{thm: ogd dyn regret}, we denote the set of stages in epoch $k$ as $\mathcal T_k=\{k\Delta+1,\dots, \min(k\Delta+\Delta,T)\}$ for $k=0, \dots, \ceil{T/\Delta}-1$; and introduce $z_k^*= \argmin_{z\in \X}\sum_{t\in \mathcal T_k}[f(z;\theta_t)]$, $y_t^*=\argmin_{x\in \X}f(x;\theta_t)$, $\bm x^*=\argmin_{\bm x\in \X^T} \sum_{t=1}^T[f(x_t;\theta_t)+ d(x_t, x_{t-1})]$. Notice that $z_k^*, y_t^*, x_t^*$ are all random variables depending on $\bm \theta$. The expected dynamic regret of OGD can be bounded as follows.
\begin{align*}
\E[\text{Reg}(OGD)]& = \sum_{t=1}^T \E[f(x_t;\theta_t)+ d(x_t, x_{t-1})]- \sum_{t=1}^T \E[f(x^*_t;\theta_t)+ d(x^*_t, x^*_{t-1})]\\
& \leq  \sum_{t=1}^T \E[f(x_t;\theta_t)+ d(x_t, x_{t-1})]- \sum_{t=1}^T \E[f(x^*_t;\theta_t)]\\
& \leq \sum_{t=1}^T \E[f(x_t;\theta_t)+ d(x_t, x_{t-1})]- \sum_{t=1}^T \E[f(y^*_t;\theta_t)]\\
& = \sum_{k=0}^{\ceil{T/\Delta}-1}\sum_{t\in \mathcal T_k} \E[f(x_t;\theta_t)+ d(x_t, x_{t-1})- f(y^*_t;\theta_t)]\\
& =   \sum_{k=0}^{\ceil{T/\Delta}-1}\sum_{t\in \mathcal T_k} \E [f(x_t;\theta_t)- f(z_k^*;\theta_t)]+\sum_{k=0}^{\ceil{T/\Delta}-1}\sum_{t\in \mathcal T_k} \E[d(x_t, x_{t-1})]\\
& \quad \quad  + \sum_{k=0}^{\ceil{T/\Delta}-1}\sum_{t\in \mathcal T_k} \E[  f(z_k^*;\theta_t)-f(y^*_t;\theta_t)]\\
& \leq  \ceil{T/\Delta}\log(\Delta+1)\frac{2G^2}{\alpha}+\frac{h^2}{\alpha} \E\|\bm\delta(\min(W,T))\|^2+ \ceil{T/\Delta} \frac{16\beta G^2}{\alpha^2}\\
& \quad \quad +\sum_{k=0}^{\ceil{T/\Delta}-1}\sum_{t\in \mathcal T_k} \E[  f(z_k^*;\theta_t)-f(y^*_t;\theta_t)]
\end{align*}
where the first inequality uses Assumption \ref{ass: d <beta (x-y)^2/2}, the second inequality uses the optimality of $y_t^*$, the last inequality follows from taking expectation on the regret bounds in Theorem \ref{thm: OGD static regret} and the fact that the OGD considered here restarts at the beginning of each epoch $k$ and repeats the stepsizes defined in Theorem \ref{thm: OGD static regret}, thus satisfying the static regret bound and switching cost bound in Theorem \ref{thm: OGD static regret} within each epoch. 

Now, it suffices to bound $\sum_{k=0}^{\ceil{T/\Delta}-1}\sum_{t\in \mathcal T_k}\E [  f(z_k^*;\theta_t)-f(y^*_t;\theta_t)]$. By the optimality of $z_k^*$, we have that
$$ \sum_{k=0}^{\ceil{T/\Delta}-1}\sum_{t\in \mathcal T_k} \E[  f(z_k^*;\theta_t)-f(y^*_t;\theta_t)]\leq  \sum_{k=0}^{\ceil{T/\Delta}-1}\sum_{t\in \mathcal T_k} \E[  f(y^*_{k\Delta+1};\theta_t)-f(y^*_t;\theta_t)].$$
We define $\E[V^k] = \sum_{t\in \mathcal T_k}\E[ \sup_{x\in \X} | f(x;\theta_t)-f(x;\theta_{t-1})|]$.\footnote{Notice that $y_t^*$ is a  random variable depending on $\theta_t$ for all $t$. Therefore, in the inequalities below, $\E[f(y^*_{k\Delta+1};\theta_t)-f(y^*_{k\Delta+1};\theta_{k\Delta+1}) ]$ can be larger than the term $\sup_{x\in \X}\E[f(x;\theta_t)-f(x;\theta_{k\Delta+1})]$, where $x$ is restricted to only deterministic variables. Thus, the expectation operator $\E$ must be outside the $\sup$ operator in our definition of the expected variation of the environment (see $\E[V^k]$ and $\E[V_T]$).} Then, for any $t\in \mathcal T_k$,  we obtain
\begin{align*}
&\E[f(y^*_{k\Delta+1};\theta_t)-f(y^*_t;\theta_t)] \\
=&	\E[f(y^*_{k\Delta+1};\theta_t)-f(y^*_{k\Delta+1};\theta_{k\Delta+1}) ]+ \E[f(y^*_{k\Delta+1};\theta_{k\Delta+1})-f(y_t^*;\theta_{k\Delta+1})]\\
&+\E[ f(y_t^*;\theta_{k\Delta+1})-f(y^*_t;\theta_t)]\\
\leq & 2\E[V^k].
\end{align*}
By summing over $t\in \mathcal T_k$ and $k=0, \dots, \ceil{T/\Delta}-1$, we obtain
$$ \sum_{k=0}^{\ceil{T/\Delta}-1}\sum_{t\in \mathcal T_k} \E[  f(z_k^*;\theta_t)-f(y^*_t;\theta_t)] \leq 2\Delta \E[V_T]$$
Similar to the proof of Corollary \ref{thm: ogd dyn regret}, by applying the bounds above  and $\Delta=\ceil{\sqrt{2T/\E[V_T]}}$, we obtain the desired bound on the expected dynamic regret of OGD for our setting, i.e.
\begin{align*}
\E[\textup{Reg}(OGD)]\leq C_1\sqrt{\E[V_T] T}\log(1+\sqrt{T/\E[V_T]})+ \frac{h^2}{\alpha}\E[\|\bm\delta(\min(W,T))\|^2]
\end{align*}
where $C_1$ is a constant factor defined in Theorem \ref{thm: ogd dyn regret}.

Consequently, by applying Theorem \ref{thm: expected regret bounds} and the bound on $\E[\|\bm\delta(W)\|^2]$ in \eqref{equ: bound on E(delta(k))}, we have the following results.
\begin{align*}
\E[\text{Reg}(RHIG)]\leq & C_2\rho^W\sqrt{\E[V_T] T}\log(1+\sqrt{T/\E[V_T]}) +\rho^W\frac{2h^2 L}{\alpha^2}  \sum_{t=0}^{\min(W,T)-1} \|R_e\|_2 (T-t)\|P(t)\|_F^2 \\
& +\zeta\sum_{t=0}^{\min(W,T)-1} \|R_e\|_2 (T-t)\|P(t)\|_F^2 \frac{\rho^t-\rho^W}{1-\rho}\\
 \leq  & C_2\rho^W\sqrt{\E[V_T] T}\log(1+\sqrt{T/\E[V_T]}) + \zeta\sum_{t=0}^{\min(W,T)-1} \|R_e\|_2 (T-t)\|P(t)\|_F^2 \frac{\rho^t}{1-\rho}
\end{align*}
by $\frac{2h^2L}{\alpha^2}-\frac{\zeta}{1-\rho}<0$, where $C_2= \frac{2L}{\alpha}C_1$.

\subsection{Proof of Theorem \ref{thm: concentration bdd}}
The proof relies on the Hanson-Wright inequality in \cite{rudelson2013hanson}.\footnote{Here we use the fact that $\|X_i\|_{\varphi}=1$ where $\|\cdot\|_{\varphi}$ is the subGaussian norm defined in \cite{rudelson2013hanson}.} 

\begin{proposition}[Hanson-Wright Inequality \cite{rudelson2013hanson}]\label{prop: hanson wright}
Consider random Gaussian vector	$\bm u=(u_1, \dots, u_n)^\top$ with $u_i$ i.i.d. following $N(0,1)$. There exists an absolute constant $c>0$,\footnote{An absolute constant refers to a quantity that does not change with anything.} such that
	$$\Pb(\bm u^\top A \bm u\geq \E [\bm u^\top A \bm u] +b)\leq \exp\left(-c \min(\frac{b^2}{\|A\|_F^2}, \frac{b}{\|A\|_2})\right),\quad \forall \ b >0$$

\end{proposition}

Now, we are ready for the proof. For any  realization of the random vectors $\{e_t\}_{t=1}^T$, our regret bound in Section \ref{sec: regret general} still holds, i.e.
\begin{align*}
\text{Reg}(RHIG)\leq \,&\rho^W \frac{2L}{\alpha} C_1\sqrt{V_T T}\log(1+\sqrt{T/V_T})+ {\frac{2L}{\alpha}}\frac{h^2}{\alpha}\rho^{W}\!\|\bm\delta(\min(W,T))\|^2\!\\
&\!+ \sum_{k=1}^{\min(W,T)}\! \!\zeta\rho^{k-1}\|\bm \delta(k)\|^2\!+\!\one_{(W>T)} \frac{\rho^T\!-\!\rho^W}{1-\rho}\zeta \|\bm \delta(T)\|^2\\
\leq  \,&{\rho^W \frac{2L}{\alpha} C_1T\log(2)}+\! \frac{2L}{\alpha}\frac{h^2}{\alpha}\rho^{W}\!\|\bm\delta(\min(W,T))\|^2\!+\! \! \sum_{k=1}^{\min(W,T)}\! \!\zeta\rho^{k-1}\|\bm \delta(k)\|^2\\
&\!+\!\one_{(W>T)} \frac{\rho^T\!-\!\rho^W}{1-\rho}\zeta \|\bm \delta(T)\|^2
\end{align*}
where we used the technical assumption that $V_T\leq T$. Let $R(W)$ denote the second regret bound. It can be verified that
$$R(W)\leq \E[\text{Regbdd}].$$

From \eqref{equ: delta(k)=M e} in the proof of Theorem \ref{thm: expected regret bounds}, we have that $\bm \delta(k)=\bm{M}_k \bm e = \bm{M}_k \bm{R_e}^{1/2} \bm u$, where $\bm u $ is a standard Gaussian vector  for $k\leq T$; and $\bm \delta(k)= \bm{M}_T \bm{R_e}^{1/2} \bm u$ for $k\geq T$. 

When $W\leq T$, we have the following formula for the regret bound $R(W)$.
\begin{align*}
R(W)=\ &  \rho^W \frac{2L}{\alpha} C_1T\log(2)+\rho^{W} \frac{2L}{\alpha}\frac{h^2}{\alpha}\|\bm\delta(W)\|^2+ \zeta \sum_{k=1}^{W}\rho^{k-1}\|\bm \delta(k)\|^2\\
= \ &\rho^W \frac{2L}{\alpha} C_1T\log(2)\\
&+\bm u^\top \underbrace{(\rho^{W} \frac{2L}{\alpha}\frac{h^2}{\alpha} \bm{R_e}^{1/2}\bm M_W^\top \bm M_W \bm{R_e}^{1/2}+\zeta \sum_{k=1}^{W}\rho^{k-1}  \bm{R_e}^{1/2}\bm M_k^\top \bm M_k \bm{R_e}^{1/2} )}_{\bm A_W}\bm u
\end{align*}

We bound $\|\bm A_W\|_F$ below.
\begin{align*}
\|\bm A_W\|_F& \leq \rho^{W} \frac{2L}{\alpha}\frac{h^2}{\alpha} \|\bm{R_e}^{1/2}\bm M_W^\top \bm M_W \bm{R_e}^{1/2}\|_F+\zeta \sum_{k=1}^{W}\rho^{k-1} \| \bm{R_e}^{1/2}\bm M_k^\top \bm M_k \bm{R_e}^{1/2}\|_F\\
& \leq \rho^{W} \frac{2L}{\alpha}\frac{h^2}{\alpha}  \| \bm M_W \bm{R_e}^{1/2}\|_F^2+\zeta \sum_{k=1}^{W}\rho^{k-1} \|  \bm M_k \bm{R_e}^{1/2}\|_F^2\\
& =  \rho^{W} \frac{2L}{\alpha}\frac{h^2}{\alpha}  \text{tr}(\bm{R_e}^{1/2}\bm M_W^\top \bm M_W \bm{R_e}^{1/2})+ \zeta \sum_{k=1}^{W}\rho^{k-1} \text{tr}(\bm{R_e}^{1/2}\bm M_k^\top \bm M_k \bm{R_e}^{1/2})\\
& =  \rho^{W} \frac{2L}{\alpha}\frac{h^2}{\alpha}  \text{tr}(\bm{R_e}\bm M_W^\top \bm M_W )+ \zeta \sum_{k=1}^{W}\rho^{k-1} \text{tr}(\bm{R_e}\bm M_k^\top \bm M_k )\\
& \leq  \rho^{W} \frac{2L}{\alpha}\frac{h^2}{\alpha} \|R_e\|_2 \text{tr}(\bm M_W^\top \bm M_W )+ \zeta \|R_e\|_2\sum_{k=1}^{W}\rho^{k-1} \text{tr}(\bm M_k^\top \bm M_k )\\
& = \rho^{W} \frac{2L}{\alpha}\frac{h^2}{\alpha} \|R_e\|_2 \|\bm M_W\|_F^2+ \zeta \|R_e\|_2\sum_{k=1}^{W}\rho^{k-1}  \|\bm M_k\|_F^2\\
& = \rho^{W} \frac{2L}{\alpha}\frac{h^2}{\alpha} \|R_e\|_2\sum_{t=0}^{W-1}(T-t)\|P(t)\|_F^2+ \zeta \|R_e\|_2\sum_{k=1}^{W}\rho^{k-1}\sum_{t=0}^{k-1}(T-t)\|P(t)\|_F^2\\
& = \rho^{W} \frac{2L}{\alpha}\frac{h^2}{\alpha} \|R_e\|_2\sum_{t=0}^{W-1}(T-t)\|P(t)\|_F^2+ \zeta \|R_e\|_2 \sum_{t=0}^{W-1}\sum_{k=t+1}^{W}\rho^{k-1} (T-t)\|P(t)\|_F^2\\
& =  \rho^{W} \frac{2L}{\alpha}\frac{h^2}{\alpha} \|R_e\|_2\sum_{t=0}^{W-1}(T-t)\|P(t)\|_F^2+ \zeta \|R_e\|_2 \sum_{t=0}^{W-1}\frac{\rho^t -\rho^W}{1-\rho} (T-t)\|P(t)\|_F^2\\
& \leq \zeta\sum_{t=0}^{W-1} \|R_e\|_2 (T-t)\|P(t)\|_F^2 \frac{\rho^t}{1-\rho}
\end{align*}
where we used \eqref{equ: bound on E(delta(k))} and $\zeta=\frac{h^2}{\alpha}+ \frac{h^2}{2L}$ and $\rho=1-\frac{\alpha}{4L}$.

When $W>T$, we have the following formula for the regret bound $R(W)$.  
\begin{align*}
R(W)=\ & \rho^W \frac{2L}{\alpha} C_1T\log(2)+ \frac{2h^2L}{\alpha^2}\rho^{W}\|\bm\delta(T)\|^2+ \zeta \sum_{k=1}^{T}\rho^{k-1}\|\bm \delta(k)\|^2+ \zeta \|\bm \delta(T)\|^2 \frac{\rho^T-\rho^W}{1-\rho}\\
= \ & \rho^W \frac{2L}{\alpha} C_1T\log(2)\\
& +\bm u^\top \underbrace{((\frac{2h^2L}{\alpha^2}\rho^{W}+ \zeta \frac{\rho^T-\rho^W}{1-\rho}) \bm{R_e}^{1/2}\bm M_T^\top \bm M_T \bm{R_e}^{1/2}+\zeta \sum_{k=1}^{T}\rho^{k-1} \bm{R_e}^{1/2}\bm M_k^\top \bm M_k \bm{R_e}^{1/2})}_{\bm A_W}\bm u
\end{align*}
Similarly, we bound $\|\bm A_W\|_F$ below.
\begin{align*}
\|\bm A_W\|_F& \leq \left(\frac{2h^2L}{\alpha^2}\rho^{W}+ \zeta \frac{\rho^T-\rho^W}{1-\rho}\right)  \|\bm{R_e}^{1/2}\bm M_T^\top \bm M_T \bm{R_e}^{1/2}\|_F+\zeta \sum_{k=1}^{T}\rho^{k-1} \| \bm{R_e}^{1/2}\bm M_k^\top \bm M_k \bm{R_e}^{1/2}\|_F\\
& \leq  \left(\frac{2h^2L}{\alpha^2}\rho^{W}+ \zeta \frac{\rho^T-\rho^W}{1-\rho}\right) \|R_e\|_2 \|\bm M_T\|_F^2+\zeta \sum_{k=1}^{T}\rho^{k-1} \|R_e\|_2 \|\bm M_k\|_F^2\\
& \leq \left(\frac{2h^2L}{\alpha^2}\rho^{W}+ \zeta \frac{\rho^T-\rho^W}{1-\rho}\right) \|R_e\|_2\sum_{t=0}^{T-1}(T-t)\|P(t)\|_F^2+\zeta \sum_{k=1}^{T}\rho^{k-1} \|R_e\|_2 \sum_{t=0}^{k-1}(T-t)\|P(t)\|_F^2\\
& \leq \left(\frac{2h^2L}{\alpha^2}\rho^{W}+ \zeta \frac{\rho^T-\rho^W}{1-\rho}\right) \|R_e\|_2\sum_{t=0}^{T-1}(T-t)\|P(t)\|_F^2+\zeta  \|R_e\|_2 \sum_{t=0}^{T-1}\frac{\rho^t-\rho^T}{1-\rho}(T-t)\|P(t)\|_F^2\\
& \leq \zeta\sum_{t=0}^{T-1} \|R_e\|_2 (T-t)\|P(t)\|_F^2 \frac{\rho^t}{1-\rho}
\end{align*}

In conclusion, for any $W\geq 1$, we have that $R(W)=\rho^W \frac{2L}{\alpha} C_1T\log(2)+ \bm u^\top \bm A_W \bm u$, and $$\|\bm A_W\|_F \leq \zeta\sum_{t=0}^{\min(W,T)-1} \|R_e\|_2 (T-t)\|P(t)\|_F^2 \frac{\rho^t}{1-\rho}.$$
Further, we have $\|\bm A_W\|_2\leq \|\bm A_W\|_F$. Therefore, by Proposition  \ref{prop: hanson wright}, we prove the concentration bound below. For any $b>0$, 
\begin{align*}
\Pb(\text{Reg}(RHIG)\geq \E[\text{Regbdd}]+b)& \leq \Pb(R(W)\geq \E[\text{Regbdd}]+b)\\
& \leq \Pb(R(W)\geq \E[R(W)]+b)\\
& =\Pb(\bm u^\top \bm A_W \bm u\geq \E [\bm u^\top \bm A_W \bm u] +b)\\
& \leq \exp\left(-c \min(\frac{b^2}{K^2}, \frac{b}{K})\right)
\end{align*}
where $K=\zeta\sum_{t=0}^{\min(T,W)-1} \|R_e\|_2 (T-t)\|P(t)\|_F^2 \frac{\rho^t}{1-\rho}$.


\section{More details of the numerical experiments}\label{append: numerical}

\paragraph{(i) The high-level planning problem.}
The parameters   are:  $e_t \sim N(0,1)$ i.i.d., $T=20$, $\alpha=1$, $\beta=0.5$, $x_0=10$, $a=4$, $\omega=0.5$, $\eta=0.5$, $\xi_t=1$, CHC's commitment level $v=3$. The regret is averaged over 200 iterations.

\paragraph{(ii) The physical tracking problem.}
Consider the second-order system
$$ \ddot x=k_1 u+g+k_2$$
where $x$ is altitude, $\dot x$ is velocity, $\ddot x$ is acceleration, etc.

Consider a discrete-time version of the system above as 
$$ \frac{x_{t+1}-2x_t+x_{t-1}}{\Delta^2}=k_1u_t-g+k_2$$
which is equivalent to
$$ u_t=\frac{1}{k_1}(\frac{x_{t+1}-2x_{t}+x_{t-1}}{\Delta^2}-(-g+k_2)).$$

Consider a cost function at stage $t$ as
$$\frac{\alpha}{2}(x_t-\theta_t)^2+ \frac{\beta}{2}u_t^2.$$
We can write the cost function in terms of $x_t$, that is,
$$\frac{\alpha}{2}(x_t-\theta_t)^2+ \frac{\beta}{2}\frac{1}{k_1^2}(\frac{x_{t+1}-2x_{t}+x_{t-1}}{\Delta^2}-(-g+k_2))^2.$$

Notice that the switching cost is not $d(x_t, x_{t-1})$ but $d(x_{t+1}, x_t, x_{t-1})$, but we still have the local coupling property of the gradients and we can still apply RHIG. 

The experiment parameters are provided below. Consider horizon 10 seconds and time discretization $\Delta=0.1$s. Let $k_1=1$, $k_2=1$, $\alpha=1$, $\beta =1\times 10^{-5}$, $x_0=1$m, $g=9.8$m/$\text{s}^2$. Let $e_t\sim N(0,0.5^2)$ i.i.d. for all $t$. Consider $d_t=0.9 \sin(0.2 t) +1$ before $t\leq 5.6$s and $d_t= 0.3 \sin(0.2 t) +1$ afterwards. Let $\gamma=0.6$, $\xi_t=1$, $\eta=1/L$ and $L\approx 2.6$.

\end{document}